\documentclass{article}

\usepackage{PRIMEarxiv}
\usepackage[utf8]{inputenc}
\usepackage[T1]{fontenc}
\usepackage{graphicx}
\usepackage{booktabs,multirow,array}
\usepackage{amsfonts,amssymb,mathtools}
\usepackage[notext]{stix} 
\usepackage{amsthm}
\usepackage[authoryear]{natbib}
\usepackage{microtype}
\usepackage{enumitem}
\usepackage{float}
\usepackage{url}
\usepackage{hyperref}
\usepackage[capitalize,noabbrev]{cleveref}

\hypersetup{
  hidelinks,
  hypertexnames=false,
  pdftitle={Mitigating Black-Box Score-Based Query Attacks via Gradient Confusion},
  pdfauthor={Ali Javeri, Javad Hezareh, Radin Shahdaei, AmirMahdi Sadeghzadeh}
}

\graphicspath{{figs/}}
\allowdisplaybreaks

\newcommand{\Pp}{\mathbb{P}}
\DeclareMathOperator{\sign}{sign}

\theoremstyle{plain}
\newtheorem{theorem}{Theorem}[section]
\newtheorem{theorem*}{Theorem}

\newtheorem{lemma}[theorem]{Lemma}

\theoremstyle{definition}
\newtheorem{definition}[theorem]{Definition}
\newtheorem{assumption}[theorem]{Assumption}
\theoremstyle{remark}

\title{Mitigating Black-Box Score-Based Query Attacks via Gradient Confusion}

\author{
  Ali Javeri \quad Javad Hezareh \quad Radin Shahdaei \quad
  AmirMahdi Sadeghzadeh\thanks{Corresponding author} \\
  Department of Computer Engineering, Sharif University of Technology \\
  Tehran, Iran \\
  \texttt{ali.javeri78@sharif.edu} \quad
  \texttt{javad.hezareh@sharif.edu} \\
  \texttt{radin.shahdaei01@sharif.edu} \quad
  \texttt{sadeghzadeh@sharif.edu}
}
\date{}

\begin{document}
\maketitle

\begin{abstract}
Score-based black-box adversarial attacks pose a serious threat to deployed machine 
learning systems by exploiting repeated model queries to craft adversarial examples, 
requiring neither access to model parameters nor architecture. Existing defenses 
predominantly rely on stochastic output perturbations, which fail to prevent determined attackers from recovering stable gradient estimates by increasing query budgets or averaging over randomness. In this work, we 
propose the \textbf{Gradient Confusion Defense (GCD)}, a retraining-free, deterministic test-time wrapper that actively shapes the probability vector returned to querying adversaries. By solving a constrained optimization problem grounded in the model's Jacobian, GCD aligns the exposed logit-gradient with a target confusion direction that is maximally misaligned with the true adversarial loss gradient---increasing the query complexity of attacks without sacrificing predictive utility or top-1 accuracy. We provide rigorous theoretical analysis demonstrating that GCD provably slows gradient-estimation-based attacks and induces systematic accept/reject errors in random search-based attacks, with the confusion effect amplifying as the fraction of defended queries increases. Extensive experiments on CIFAR-10 and ImageNet across multiple architectures (\mbox{ResNet-18}, \mbox{ResNet-50}, \mbox{ViT-B/16}) and four state-of-the-art attacks confirm these predictions: GCD achieves state-of-the-art robustness, while incurring zero clean accuracy degradation. 
Our code is available at: \url{https://github.com/awlliom/GradientConfusion}.
\end{abstract}


\section{Introduction}

Modern machine learning systems are often deployed via remote interfaces
(e.g., cloud APIs and online services), where users interact with models only
through query access. In this setting, standard white-box threat assumptions
\citep{Szegedy2014intriguing, Goodfellow2015explaining, MoosaviDezfooli2016deepfool, carlini2017towards, Athalye2018obfuscated, Madry2018pgd}
are frequently unrealistic, yet query-based black-box attacks remain highly
effective even without access to model parameters, gradients, or training data
\citep{Ilyas2018nes, AlDujaili2020signattack, Andriushchenko2020square, Cheng2020signopt, Guo2019simple, Ilyas2019bandit, Vo2024brusle, Reza2025gsba, Chen2020rays, Chen2020hsja, Reza2023cgba}.
By adaptively probing the model and observing outputs, adversaries can craft
adversarial examples under practical query budgets, making black-box robustness
a core requirement for secure and trustworthy deployment.

Query-based black-box attacks are commonly divided into \emph{decision-based} attacks~\citep{Cheng2020signopt, Chen2020rays, Chen2020hsja, Reza2023cgba}, which rely only on predicted labels and search for the decision boundary, and \emph{score-based} attacks~\citep{Andriushchenko2020square, AlDujaili2020signattack, Guo2019simple, Ilyas2019bandit, Vo2024brusle}, which use logits or class probabilities to estimate gradients via finite-difference or search-based strategies. In practice, score-based attacks are typically more query-efficient and achieve higher success rates, posing a serious risk to deployed classifiers that must return confidence scores. As a result, a large body of work~\citep{Qin2021rnd, Xie2018mitigate, He2019pni, Chen2022aaa, Aithal2022ornd, dashtbani2025rls, HungQuang2024rfd, wang2021fighting} has proposed defenses, most of which introduce randomness into model responses~\citep{Qin2021rnd, Aithal2022ornd, dashtbani2025rls, HungQuang2024rfd}. While such methods can increase variance or slow optimization, they do not explicitly control the gradient signal exposed through outputs, allowing attackers to average over randomness~\citep{Athalye2018eot} or increase the query budget to recover stable estimates when meaningful scores must be preserved.

In this work, we revisit black-box defense from a different perspective and
ask: \emph{What gradient signal should a model expose to an adversary?}
Instead of relying on stochastic perturbations, we propose to actively
manipulate the gradient information inferred from model outputs. Our goal is
not merely to add noise, but to deliberately bias the gradient estimates or search directions that score-based attacks attempt to recover, steering them away from directions that reduce the classification margin.

We study the problem of defending a fixed, pretrained target classifier against
query-based black-box adversaries while preserving correct predictions and
usable score outputs. We introduce a deterministic test-time response wrapper
that modifies the returned probability vector without altering the classifier
itself or its training procedure. The core idea is \emph{gradient confusion}:
we deliberately shape the output scores so that the induced adversarial loss gradient's
direction deviates from the true adversarial loss gradient that score-based attacks seek
to estimate. Using the loss gradients and Jacobian of the model’s logits, we derive an analytic target direction that misaligns the attacker’s search direction or estimated gradient, thereby increasing the number of queries required for
successful adversarial example generation. 

Building on a target confusion direction, we introduce the \emph{Gradient Confusion Defense (GCD)}, which formulates a constrained optimization problem over output probabilities to explicitly control the gradient signal exposed to score-based black-box attackers. GCD enforces three principled constraints at test time: exact preservation of the top-1 prediction to maintain clean accuracy, a bounded $\ell_1$ distortion on the returned probability vector to preserve the utility of confidence scores, and alignment of the exposed logit-gradient with a target confusion direction rather than the true gradient.

The main contributions of this work are threefold.
\begin{itemize}[itemsep=2pt, topsep=2pt, parsep=0pt, partopsep=0pt]
	\item We propose the \emph{Gradient Confusion Defense (GCD)}, a practical, retraining-free test-time defense that explicitly controls the gradient signal exposed through model outputs in the score-based black-box threat model.
	\item We \emph{theoretically} analyze how GCD forces the adversary to issue more queries to generate a successful adversarial example using gradient estimation-based and search-based attacks.
	\item \emph{Extensive experiments} show that GCD substantially reduces attack success rates while preserving clean accuracy, and further validate our theoretical analysis and demonstrate the effectiveness of our defense against several state-of-the-art query-based and adaptive attacks.
\end{itemize}

\section{Related Works}

\textbf{Score-based query attacks.} Query-based black-box adversarial attacks craft adversarial examples using only repeated queries to a deployed model, without access to parameters or gradients. These attacks are typically divided into \emph{score-based} methods~\citep{Andriushchenko2020square, AlDujaili2020signattack, Guo2019simple, Ilyas2019bandit, Vo2024brusle, Reza2025gsba, Ilyas2018nes, Chen2017zoo}, which leverage logits or class probabilities, and \emph{decision-based} methods~\citep{Cheng2020signopt, Chen2020rays, Chen2020hsja, Reza2023cgba}, which use only predicted labels and are often less query-efficient.
Early score-based attacks estimated gradients via zero-order optimization and finite-difference style approximations~\citep{Chen2017zoo, Ilyas2018nes, Ilyas2019bandit}, while later work emphasized search-based strategies that can be highly query-efficient~\citep{Andriushchenko2020square, AlDujaili2020signattack}.
Bandit-TD~\citep{Ilyas2019bandit} frames gradient estimation as a bandit optimization problem and exploits temporal correlation of successive gradients and spatial structure in images.
SquareAttack~\citep{Andriushchenko2020square} performs randomized search using localized square-shaped perturbations, accepting only loss-decreasing proposals, yielding strong practical query efficiency.
BruSLeAttack~\citep{Vo2024brusle} targets sparse ($\ell_0$) perturbations by searching over a reduced set of pixel locations to replace with values from a fixed synthetic image, and uses a Bayesian strategy guided by historical pixel-manipulation outcomes to improve query efficiency.

\paragraph{Black-box defenses.} A foundational line of defense is adversarial training (AT)~\citep{Tramer2018at, Salman2019robustrandomizedsmoothing, Rice2020at}, which empirically improves robustness by training on adversarially perturbed samples. Extensions incorporate data augmentation~\citep{Rebuffi2021dataaugmentation}, denoising~\citep{Salman2020denoised}, and ensembling~\citep{Liu2018selfensemble}, but AT often remains costly and can reduce clean accuracy; moreover, many such defenses primarily target white-box threats and can be brittle under obfuscated-gradient effects~\citep{Athalye2018obfuscated}.
For query-based threats, defenses commonly aim to disrupt the attacker’s ability to recover stable gradients from outputs, typically via randomization or output/logit manipulation.
AAA~\citep{Chen2022aaa} explicitly optimizes returned logits to confound the attacker by wrapping the original loss with a surrogate that preserves similar degradation behavior while inducing opposite gradient slopes, thereby steering optimization toward incorrect directions.
RLS~\citep{dashtbani2025rls} randomizes the effective scale of logits by multiplying them with a random coefficient, which empirically degrades gradient estimation and slows attacks.
Boundary defense~\citep{Aithal2022ornd} injects noise into the model’s output to hinder convergence, and RND~\citep{Qin2021rnd} injects noise at the input level, both aiming to increase the query complexity required for successful attacks.

\section{Preliminaries}

	\label{parag:define}
	Let $\mathcal{X} \subseteq [0,1]^d$ be the input space, and let $[K] = \{1,\dots,K\}$ denote the label set.
    A classifier $\mathcal{M} : \mathcal{X} \to \Delta^K$ outputs a probability vector
    $\mathcal{M}(\mathbf{x}) = \sigma(\mathbf{z}(\mathbf{x}))$, where $\mathbf{z}(\mathbf{x})$ denotes the logits and $\sigma$ is the softmax function.
    Here, $\Delta^K = \{ \mathbf{p} \in \mathbb{R}^K \mid \sum_i \mathbf{p}_i = 1 \}$ denotes the probability simplex.
    Let $k \in [K]$ denote the true label of input $\mathbf{x}$.
    For simplicity, we denote $\hat{\mathbf{y}} := \mathcal{M}(\mathbf{x})$.

	
	\paragraph{Attacks and perturbations.}
	An adversarial example $\mathbf{x}_{\mathrm{adv}}$ for a sample $\mathbf{x}$ with corresponding label $k$ is generated by adding a perturbation $\boldsymbol{\delta}$ to $\mathbf{x}$.
    Typically, the norm of the perturbation $\boldsymbol{\delta}$ is bounded so as not to change the ground-truth label of $\mathbf{x}$.
    Formally, an adversarial example is defined as
    $\mathbf{x}_{\mathrm{adv}} = \mathbf{x} + \boldsymbol{\delta}$,
    such that $\lVert \mathbf{x}_{\mathrm{adv}} - \mathbf{x} \rVert_p \le R$, where $R$ is the perturbation budget, and the classifier $\mathcal{M}$ misclassifies $\mathbf{x}_{\mathrm{adv}}$
    (untargeted attack: $\arg\max_i \mathcal{M}(\mathbf{x}_{\mathrm{adv}})_i \neq k$)
    or predicts a target class $t$
    (targeted attack: $\arg\max_i \mathcal{M}(\mathbf{x}_{\mathrm{adv}})_i = t$).

	
	There are two strategies for generating adversarial examples in the black-box setting: transfer-based \citep{Papernot2017transfer} and query-based \citep{Chen2017zoo}.
    In the query-based setting (our focus), the adversary can adaptively query the model up to a budget $Q$ that bounds the number of query the adversary can perform.
    Each query returns either (i) a score or probability vector (score-based setting) or (ii) only the predicted label (decision-based setting).
    We focus on \textit{score-based} black-box attacks, as they are generally stronger than decision-based attacks and harder to defend \citep{Ilyas2018nes}.
    Query-based black-box attacks typically optimize an adversarial loss using only the model's outputs.

	
	\paragraph{Adversarial Loss.}
	To attack the black-box target classifier $\mathcal{M}(\mathbf{x})$,
    the following adversarial loss function has been widely used in score-based attacks \citep{Chen2017zoo, Ilyas2018nes, Qin2021rnd}:
    \begin{equation}
    \label{eq:margin_def1}
    f(\mathbf{x}) := \mathcal{M}(\mathbf{x})_i - \mathcal{M}(\mathbf{x})_j,
    \end{equation}
    where $i$ and $j$ denote the indices of the largest and second-largest components of the probability vector $\mathcal{M}(\mathbf{x})$, respectively.
    Decreasing $f(\mathbf{x})$ encourages the top prediction (typically the correct label) to lose confidence relative to its strongest competitor.
    There are two common approaches for minimizing $f(\mathbf{x})$: gradient estimation (Zero-Order Optimization) and random search.

	
	\paragraph{Gradient estimation (Zero-Order Optimization)}
	The most prominent approach for generating adversarial examples is to update the input in the opposite direction of the gradient of the adversarial loss function with respect to the input, $\nabla_{\mathbf{x}} f(\mathbf{x})$, so as to decrease the adversarial loss.
    Since the attacker cannot directly access the model's parameters or the gradient $\nabla_{\mathbf{x}} f(\mathbf{x})$, it instead estimates the gradient using only the model's outputs.
    A gradient estimator that is widely used in score-based attacks \citep{Ilyas2018nes, nesterov2017random, Qin2021rnd} is given by
    \begin{equation}
    \label{eq:zo}
    \widehat{\nabla}_{\mathbf{x}} f_\mu(\mathbf{x})
    := \frac{1}{m} \sum_{i=1}^m
    \frac{f(\mathbf{x} + \mu \mathbf{u}_i) - f(\mathbf{x})}{\mu}\,\mathbf{u}_i,
    \end{equation}
    where $\mathbf{u}_i \sim \mathcal{N}(\mathbf{0}, \mathbf{I})$ and $\mu > 0$ is a small smoothing parameter.
    In this setup, the adversary queries the model at points in a neighborhood around the input $\mathbf{x}$. As a result, they effectively observe a Gaussian-smoothed version of the loss function, denoted by $f_\mu(\mathbf{x})$.
    The attacker then performs Projected Gradient Descent (PGD) steps based on $\widehat{\nabla}_{\mathbf{x}} f_\mu(\mathbf{x})$ while respecting the input perturbation constraint:
    \begin{equation}
    \label{eq:pgd}
    \mathbf{x}_{t+1} = \Pi_R\!\left(\mathbf{x}_t - \eta_t \widehat{\nabla}_{\mathbf{x}} f_\mu(\mathbf{x}_t)\right),
    \end{equation}
    where $\eta_t$ is step size at iteration t, and $\Pi_\xi$ denotes projection onto the $\ell_p$-ball of radius $R$, i.e., $\| \mathbf{x}' - \mathbf{x} \|_p \le R$. Each step of PGD requires one query to the target model. Increasing the number of steps required for PGD to converge increases the query complexity of the attack.

	
    \paragraph{Random-search.}
Random search is an iterative method for minimizing the adversarial loss function.
Let $\mu>0$ denote the step size (i.e., the perturbation magnitude used by the search rule).
At iteration $t$, given the current iterate $\mathbf{x}_t$, the attacker samples a search direction
$\mathbf{u}_t \sim \mathcal{D}$ from a predefined distribution (e.g., Gaussian directions or structured
perturbations), and evaluates the candidate $\mathbf{x}_t + \mu \mathbf{u}_t$.
If this perturbation decreases the adversarial loss $f(\cdot)$ in~\eqref{eq:margin_def1}, the attacker
updates $\mathbf{x}_{t+1} \leftarrow \mathbf{x}_t + \mu \mathbf{u}_t$; otherwise, $\mathbf{x}_{t+1} \leftarrow \mathbf{x}_t$.
This type of random-search (hill-climbing) rule is commonly used in query-based black-box attacks
\citep{Guo2019simple,Andriushchenko2020square,Ilyas2019bandit}.
Using the adversarial loss $f(\mathbf{x})$, the one-step improvement signal is defined as
\begin{equation}
\label{eq:rs-h}
h_t := f(\mathbf{x}_t + \mu \mathbf{u}_t) - f(\mathbf{x}_t).
\end{equation}

    If $h_t < 0$, then moving along $\mathbf{u}_t$ decreases the loss and is therefore considered a useful attack direction.
    Accordingly, the accepted update is given by
    \begin{equation}
    \label{eq:rs-update}
    \mathbf{x}_{t+1} = \Pi_\xi\!\left(\mathbf{x}_t + \mathbb{I}\!\left(h_t < 0\right)\cdot \mu \mathbf{u}_t\right),
    \end{equation}
    where $\mathbb{I}(\cdot)$ denotes the indicator function.

\section{Gradient Confusion Defense (GCD)}

In score-based black-box attacks, the adversary minimizes the loss $f(\mathbf{x})$
in \eqref{eq:margin_def1} using only the model output $\hat{\mathbf{y}}\in\Delta^K$.
Gradient Confusion Defense (GCD) disrupts this process by returning a modified
output $\tilde{\mathbf{y}}\in\Delta^K$ instead of $\hat{\mathbf{y}}$, increasing
the number of queries needed to succeed (potentially exceeding the budget $Q$).
GCD constructs $\tilde{\mathbf y}$ from $\mathcal{M}$ and $\hat{\mathbf y}$ while
preserving utility by keeping $\|\hat{\mathbf y}-\tilde{\mathbf y}\|_1$ small and
ensuring $\text{argmax}_k(\hat{\mathbf y}_k)=\text{argmax}_k(\tilde{\mathbf y}_k)$.
For analysis, we introduce a hypothetical model $\tilde{\mathcal{M}}$ with
$\tilde{\mathbf y}=\tilde{\mathcal{M}}(\mathbf x)=\sigma(\tilde{\mathbf z}(\mathbf x))$,
and drop the dependence on $\mathbf x$ in notation.

Since the adversary only sees the output of the defended target model by the gradient confusion method, their adversarial loss function transforms to: \begin{equation} \label{eq:margin_def} \tilde{f}(\mathbf{x}) := \tilde{\mathcal{M}}(\mathbf{x})_i - \tilde{\mathcal{M}}(\mathbf{x})_j, \end{equation}

where $i$ and $j$ denote the indices of the largest and second-largest components of the probability vector $\tilde{\mathcal{M}}(\mathbf{x})$, respectively. The gradient confusion solves the following minimization problem to compute $\tilde{\mathbf y}$: \begin{equation} \label{eq:mainproblem} \begin{aligned} &\underset{\tilde{y}}{\arg\min} \: \: \; \langle \nabla_{\mathbf x} f(\mathbf x), \nabla_{\mathbf x} \tilde{f}(\mathbf x) \rangle, \\ \text{s.t} \quad \| &\mathcal{M}(\mathbf x) - \tilde{\mathcal{M}}(\mathbf x) \|_1 \leq \xi \quad \text{and} \quad \tilde{\mathbf y} \in \Delta^K \end{aligned} \end{equation}

where $\xi$ is the distortion budget. The GCD minimization problem addresses two objectives.
\textbf{(i) \textit{Confusion}}: Minimizing $\langle \nabla_\mathbf{x} f(\mathbf{x}), \nabla_\mathbf{x} \tilde{f}(\mathbf{x}) \rangle$ misaligns the adversary’s estimated gradient $\nabla_\mathbf{x}\tilde{f}(\mathbf{x})$ from the true gradient $\nabla_\mathbf{x} f(\mathbf{x})$. This increases the query cost of gradient-estimation attacks and can mislead random-search attacks by flipping the sign of $h_t$. Theorems \ref{thm:prob-converge} and \ref{thm:rs_switch_failure} in Section \ref{subsec:theoryanalys} provide theoretical support, and Section \ref{sec:eval_sadeghzadeh} validates GCD empirically.
\textbf{(ii) \textit{Utility}}: The constraint $|| \mathcal{M}(\mathbf{x}) - \tilde{\mathcal{M}}(\mathbf{x}) ||_1 \leq \xi$ (equivalently $\| \hat{\mathbf{y}} - \tilde{\mathbf{y}}\|_1 \leq \xi)$ keeps $\tilde{\mathbf{y}}=\tilde{\mathcal{M}}(\mathbf{x})$ close to $\hat{\mathbf{y}}=\mathcal{M}(\mathbf{x})$, preserving predictive performance. Section \ref{subsec:method} details how we solve \eqref{eq:mainproblem}.

\subsection{Theoretical Analysis}
\label{subsec:theoryanalys}
In this section, we theoretically show that solving the optimization problem in \eqref{eq:mainproblem} increases the model's resistance to score-based black-box attacks.
To incorporate the crafted loss function $\tilde{f}$, we define a randomized defended objective:
\begin{equation}
\label{eq:f_def}
f_{\mathrm{def}}(\mathbf{x})
:=
\lambda f(\mathbf{x}) + (1-\lambda)\tilde f(\mathbf{x}),
\quad
\lambda \sim \mathrm{Bernoulli}(p),
\end{equation}
where $\lambda$ is drawn independently for each query.
This theoretical model allows us to formalize the effect of $\tilde{f}$ on the optimization problem faced by the adversary.

Under mild assumptions, we show that using $f_{\text{def}}$ increases the number of steps required by gradient-estimation-based attacks and misleads search-based attacks.
Consequently, it confuses the adversary and increases the number of queries required to craft an adversarial input.
Moreover, we argue that setting $p \to 0$ maximizes this confusion effect.

\begin{assumption}
\label{assum:lipschitz}
The loss functions $f$ and $\tilde f$ are Lipschitz continuous with constants $L_0(f)$ and $L_0(\tilde f)$, respectively.
\end{assumption}

\begin{assumption}
\label{assum:smooth}
The loss functions $f$ and $\tilde f$ are smooth, meaning their gradients are Lipschitz continuous with constants $L_1(f)$ and $L_1(\tilde f)$, respectively.
\end{assumption}

\subsubsection{Gradient Estimation-Based Attacks}
Gradient estimation-based attacks approximate gradients using function evaluations (e.g., \eqref{eq:zo}) and update the input via projected gradient descent (e.g., \eqref{eq:pgd}).
Thus, the adversary applies (projected) gradient descent to minimize the adversarial loss.
The number of queries performed by the adversary is denoted by $T$, which corresponds to the number of optimization steps.
In theory, the adversary is considered to have succeeded in optimizing the adversarial loss function if the norm of its gradient is sufficiently small at the final iterate.
Thus, to analyze the adversary's optimization procedure, we need to bound the gradient norm.
In our framework, using \eqref{eq:f_def}, the gradient estimator is driven by $f_{\mathrm{def}}$ rather than $f$.

\begin{theorem}
\label{thm:prob-converge}
Suppose Assumption~\ref{assum:lipschitz} holds.
Consider zero-order optimization using the gradient estimator driven by $f_{\text{def}}$, with smoothing parameter $\mu \le \mu^\ast$, for $T$ iterations and constant step size $\eta_t = \eta^\ast$.
If, for all $\mathbf{x} \in \mathbb{R}^d$, it holds that $|f(\mathbf{x}) - \tilde{f}(\mathbf{x})| \le \epsilon$, where $\epsilon > 0$ is a small constant, then the following bound holds:
\begin{equation}
\label{eq:prob-converge}
\begin{aligned}
& \frac{p}{T} \sum_{t=0}^{T-1} \mathbb{E}\!\left[\|\nabla f_\mu(\mathbf{x}_t)\|^2\right]
+ \frac{1-p}{T} \sum_{t=0}^{T-1} \mathbb{E}\!\left[\left\langle \nabla f_\mu(\mathbf{x}_t), \nabla \tilde{f}_\mu(\mathbf{x}_t) \right\rangle\right] \\
& \leq
\sqrt{\frac{4d^3 L_0(f)^5 R}{\epsilon^\prime T}
\left(
p + (1-p)\frac{L_0(\tilde{f})^2}{L_0(f)^2}
+ 2p(1-p)\frac{\epsilon^2}{\epsilon^{\prime 2}}
\right)}.
\end{aligned}
\end{equation}
Here, $R$ denotes the perturbation budget, $\epsilon^\prime > 0$ is a small constant ensuring that $f_\mu$ is close to $f$, and the expectation is taken over the randomness of the optimization process up to iteration $t$.
\end{theorem}

The detailed proof of Theorem~\ref{thm:prob-converge} is provided in Appendix~\ref{apx:conv-theorem-proof}.
This theorem shows that incorporating the function $\tilde{f}$ into the optimization process can effectively perturb the optimization landscape and mislead the gradient estimates.
In particular, by optimizing the confusion term $\langle \nabla f_\mu(\mathbf{x}_t), \nabla \tilde{f}_\mu(\mathbf{x}_t) \rangle$, we can increase the overall bound and thereby improve robustness against gradient estimation attacks.

\paragraph{Implications of Theorem~\ref{thm:prob-converge}.} Let {\scriptsize\[
C(p) = \sqrt{\frac{4d^3 L_0(f)^5 R}{\epsilon^\prime}\left(p + (1-p)\frac{L_0(\tilde{f})^2}{L_0(f)^2}
+ 2p(1-p)\frac{\epsilon^2}{\epsilon^{\prime 2}} \right)},
\]}be a constant depending on $p$ and $C(0) \ne 0$.
Rearranging the bound in \eqref{eq:prob-converge} as:
\begin{equation}
\label{eq:our-bound}
    \begin{aligned}
    &\frac{1}{T} \sum_{t=0}^{T-1} \mathbb{E}\!\left[\| \nabla f_\mu(\mathbf{x}_t) \|^2\right]
    \leq \frac{C(p)}{p\sqrt{T}} \\
    &\qquad - \frac{(1-p)}{p}\,
    \frac{1}{T} \sum_{t=0}^{T-1}
    \mathbb{E}\!\left[\langle \nabla f_\mu(\mathbf{x}_t),
    \nabla \tilde{f}_\mu(\mathbf{x}_t) \rangle\right]
    \end{aligned}
\end{equation}
For adversary to succeed in optimization, they need to minimize the right hand side of the inequality below a certain threshold by increasing $T$ which is the number of optimization steps. Thus the defender can increase the second term in the right hand side by optimizing the confusion term $\langle \nabla f_\mu(\mathbf{x}_t), \nabla \tilde{f}_\mu(\mathbf{x}_t) \rangle$ to be negative and large in magnitude. In other words, the defender needs to make the gradients of the two functions to be misaligned. This would force the attacker to use a larger number of optimization steps (queries) $T$ to achieve the desired threshold, thus increasing the attack cost.

\paragraph{Choosing $p \to 0$.}
We can further increase the bound in \eqref{eq:our-bound} by choosing $p$ close to $0$. When $p \to 0$, the factor $(1-p)/p$ amplifies the contribution of the confusion term. Therefore, if the defender enforces $\langle \nabla f_\mu(\mathbf{x}_t), \nabla \tilde{f}_\mu(\mathbf{x}_t) \rangle$ to be negative and large in magnitude, the right-hand side of \eqref{eq:our-bound} can become arbitrarily large, effectively preventing the attacker from driving the average gradient norm below a fixed threshold within a limited number of iterations. This motivates choosing $p=0$, i.e., relying entirely on $\tilde{f}$ in the optimization process.

\subsubsection{Search-Based Attacks}
\label{sec:rs_switch_analysis}

Random-search score-based attacks propose a direction $\mathbf{u}_t$ and accept or reject it based on the \emph{sign} of the one-step improvement signal (cf.\ \eqref{eq:rs-h} and \eqref{eq:rs-update}). Under our switching defense, each query is answered using the randomized objective $f_{\mathrm{def}}$, which independently returns either $f$ or $\tilde f$; thus the two evaluations used to form the improvement from $\mathbf{x}_t$ to $\mathbf{x}_t+\mu\mathbf{u}_t$ may come from \emph{different} objectives. Let $h_{\mathrm{def},t}$ denote the improvement signal computed from the defender’s returned values. Because $h_{\mathrm{def},t}$ can mix $f$ and $\tilde f$ across the two points, $\sign(h_{\mathrm{def},t})$ can disagree with $\sign(h_t)$, causing the attacker to accept harmful steps or reject helpful ones. The next result lower bounds the probability of such one-step sign errors via (i) a geometric term depending on the angle between $\nabla f(\mathbf{x})$ and $\nabla \tilde f(\mathbf{x})$, and (ii) a smoothness correction scaling with $\mu$ (vanishing as $\mu\to 0$ under Assumption~\ref{assum:smooth}).

\begin{theorem}
\label{thm:rs_switch_failure}
Let Assumption~\ref{assum:smooth} hold. Let $d\ge 2$ and $\mathbf{u}\sim\mathcal N(\mathbf{0},\mathbf{I})$.
Fix $\mathbf{x}\in\mathbb R^d$ and $\mu>0$. Define the true and defended one-step signals
\[
h := f(\mathbf{x}+\mu\mathbf{u})-f(\mathbf{x}),
\quad
h_{\mathrm{def}}
:=
f_{\mathrm{def}}(\mathbf{x}+\mu\mathbf{u})-f_{\mathrm{def}}(\mathbf{x}).
\]
Then
\begin{equation}\label{eq:rs_switch_failure_bound}
\begin{aligned}
 \mathbb{P}&\bigl(\sign(h)\neq \sign(h_{\rm def})\bigr) \geq \\
& (1-p)^2\Bigg[
\frac{1}{\pi}\arccos\Biggl(
\frac{\langle \nabla f(\mathbf{x}),\nabla \tilde f(\mathbf{x})\rangle}
{\|\nabla f(\mathbf{x})\|\,\|\nabla \tilde f(\mathbf{x})\|}
\Biggr) \\
& - \frac{\mu d}{\sqrt{2 \pi}}
\left(
\frac{L_1(f)}{\| \nabla f(\mathbf{x}) \|} + 
\frac{L_1(\tilde f)}{\| \nabla \tilde f(\mathbf{x}) \|}
\right)
\Bigg]
\end{aligned}
\end{equation}

\end{theorem}

The proof of Theorem~\ref{thm:rs_switch_failure} is deferred to Appendix~\ref{app:search-based}. The event $\sign(h)\neq \sign(h_{\mathrm{def}})$ is exactly a \emph{one-step decision error}: relative to the true objective $f$, the attacker either rejects a truly improving direction or accepts a non-improving direction. Theorem~\ref{thm:rs_switch_failure} shows that this error probability is lower bounded by a misalignment term, proportional to the angle between $\nabla f(\mathbf{x})$ and $\nabla \tilde f(\mathbf{x})$, up to a smoothness correction that decreases with the step size $\mu$. Consequently, designing $\tilde f$ so that $\nabla \tilde f(\mathbf{x})$ is strongly misaligned with $\nabla f(\mathbf{x})$ forces random search to waste queries, degrading progress under a fixed budget.  A discussion of the Gaussian direction model (and its relationship to isotropic samplers used in practice) is provided in Appendix~\ref{app:gaussian}

\paragraph{Conclusion.}
In summary, this subsection provides theoretical support that GCD increases the query cost of score-based black-box attacks. Theorem~\ref{thm:prob-converge} shows that inducing a negative confusion term slows gradient-estimation-based optimization, while Theorem~\ref{thm:rs_switch_failure} shows that gradient misalignment causes accept/reject errors in search-based attacks. Together, these results justify designing $\tilde f$ to remain close to $f$ in value while being misaligned in gradient.

\subsection{Solving the GCD Minimization Problem}
\label{subsec:method}
GCD is a \emph{test-time output wrapper} defense. It leaves the underlying target classifier unchanged,
but modifies the probability vector returned to users. Concretely, for an input $\mathbf{x}$, the target
model produces logits $\mathbf{z}(\mathbf{x})\in\mathbb{R}^K$ and probability vector
$\hat{\mathbf{y}}=\sigma(\mathbf{z}(\mathbf{x}))\in\Delta^{K}$. We suppose that GCD returns a perturbed distribution
$\tilde{\mathbf{y}}=\sigma(\tilde{\mathbf{z}}(\mathbf{x}))\in\Delta^{K}$. We first focus on minimizing
$\langle \nabla_{\mathbf{x}} f(\mathbf{x}), \nabla_{\mathbf{x}} \tilde{f}(\mathbf{x}) \rangle$, and then incorporate the constraints
$\| \mathcal{M}(\mathbf{x}) - \tilde{\mathcal{M}}(\mathbf{x}) \|_1 \leq \xi$ and $\tilde{\mathbf{y}}\in \Delta^K$.

Based on the chain rule, we can write the inner product as follows:
\begin{equation}
\label{eq:gc1}
\begin{aligned}
	\underset{\tilde{\mathbf{y}}}{\arg\min} \; & \left\langle \nabla_{\mathbf{x}} f(\mathbf{x}), \nabla_{\mathbf{x}} \tilde{f}(\mathbf{x}) \right\rangle \\
	= \underset{\tilde{\mathbf{y}}}{\arg\min} \; & \left\langle \nabla_{\mathbf{x}} \mathbf{z}(\mathbf{x}) \,\nabla_{\mathbf{z}} f(\mathbf{z}(\mathbf{x})), \;
	\nabla_{\mathbf{x}} \tilde{\mathbf{z}}(\mathbf{x}) \,\nabla_{\tilde{\mathbf{z}}} \tilde{f}(\tilde{\mathbf{z}}(\mathbf{x})) \right\rangle .
\end{aligned}
\end{equation}

Solving \eqref{eq:gc1} exactly is difficult because $\tilde{\mathbf{y}}=\sigma(\tilde{\mathbf{z}}(\mathbf{x}))$ appears \emph{inside} the objective through $\tilde{f}$ and $\nabla_{\mathbf{x}} \tilde{f}$, so the optimal $\tilde{\mathbf{z}}$ determines the gradient terms needed to optimize for $\tilde{\mathbf{z}}$ (a circular \textit{chicken-and-egg} dependency). Therefore, due to the constraint $\| \mathcal{M}(\mathbf{x}) - \tilde{\mathcal{M}}(\mathbf{x}) \|_1 \leq \xi$, we use a local approximation of $\tilde{\mathbf{z}}(\mathbf{x})$ by $\mathbf{z}(\mathbf{x})$ (i.e., $\mathbf{z}(\mathbf{x}) \approx \tilde{\mathbf{z}}(\mathbf{x})$), which yields $\nabla_{\mathbf{x}} \mathbf{z}(\mathbf{x}) \approx \nabla_{\mathbf{x}} \tilde{\mathbf{z}}(\mathbf{x})$. Letting $J(\mathbf{x}) = \nabla_{\mathbf{x}} \mathbf{z}(\mathbf{x})$, our objective then approximately becomes:
\begin{equation}
\label{eq:gc2}
\begin{aligned}
	\underset{\tilde{\mathbf{y}}}{\arg\min} \; & \left\langle \nabla_{\mathbf{x}} f(\mathbf{x}), \nabla_{\mathbf{x}} \tilde{f}(\mathbf{x}) \right\rangle \\
	\approx \; \underset{\tilde{\mathbf{y}}}{\arg\min} \: &
	\left\langle J(\mathbf{x})\,\nabla_{\mathbf{z}} {f}(\mathbf{z}(\mathbf{x})),\;
	J(\mathbf{x})\,\nabla_{\mathbf{z}} \tilde{f}(\mathbf{z}(\mathbf{x})) \right\rangle .
\end{aligned}
\end{equation}

Our experiments show this surrogate delivers consistent improvements against query-based black-box attacks.
The only derivative that is a function of $\tilde{\mathbf{y}}$ is $\nabla_{\mathbf{z}} \tilde{f}(\mathbf{z}(\mathbf{x}))$.
Let $\mathbf{z}_k(\mathbf{x})$ denote the $k$-th component of the perturbed logits vector $\mathbf{z}(\mathbf{x})$. For the $k$-th component of
$\nabla_{\mathbf{z}} \tilde{f}(\mathbf{z}(\mathbf{x}))$, we have
\begin{equation}
\frac{\partial \tilde{f}(\mathbf{z}(\mathbf{x}))}{\partial \mathbf{z}_k(\mathbf{x})} =
\begin{cases}
	\tilde{\mathbf{y}}_i(1 - \tilde{\mathbf{y}}_i) - (- \tilde{\mathbf{y}}_j \tilde{\mathbf{y}}_i), & \text{if } k = i, \\
	-\tilde{\mathbf{y}}_i \tilde{\mathbf{y}}_j - \tilde{\mathbf{y}}_j(1 - \tilde{\mathbf{y}}_j), & \text{if } k = j, \\
	-\tilde{\mathbf{y}}_i \tilde{\mathbf{y}}_k - (-\tilde{\mathbf{y}}_j \tilde{\mathbf{y}}_k), & \text{if } k \neq i, k \neq j.
\end{cases}
\end{equation}
where $i = \arg\max_k \tilde{\mathbf{y}}_k$ (index of predicted class after defense), and
$j = \arg\max_{k \neq i} \tilde{\mathbf{y}}_k$ (index of runner-up class after defense).
Since we do not know the index of the predicted class and the runner-up class of $\tilde{\mathbf{y}}$ before minimization, we use the index of the predicted class and the runner-up class of $\hat{\mathbf{y}}$ to construct $\nabla_{\mathbf{z}} \tilde{f}(\mathbf{z}(\mathbf{x}))$.
Hence, $i = \arg\max_k \hat{\mathbf{y}}_k$ and $j = \arg\max_{k \neq i} \hat{\mathbf{y}}_k$.

The optimal solution for the minimization problem \eqref{eq:gc2} with respect to $\nabla_{\mathbf{z}}^{*}\tilde{f}(\mathbf{z}(\mathbf{x}))$ is as follows:
\begin{equation}
\label{eq:optimal_grad}
\nabla_{\mathbf{z}}^{*}\tilde{f}(\mathbf{z}(\mathbf{x}))
=
-\alpha\, J(\mathbf{x})^{\top} J(\mathbf{x})\,\nabla_{\mathbf{z}} f(\mathbf{z}(\mathbf{x})) .
\end{equation}

where $\alpha \rightarrow \infty$. In our experiments, we found that the norm of $\nabla_{\mathbf{z}}^{*} \tilde{f}(\mathbf{z}(\mathbf{x}))$ is not important for the success of GCD. Thus, we choose $\alpha = 1$. Since we have $J(\mathbf{x})$ and $\nabla_{\mathbf{z}} f(\mathbf{z}(\mathbf{x}))$, we can compute $\nabla_{\mathbf{z}}^{*} \tilde{f}(\mathbf{z}(\mathbf{x}))$. Because the equation
$\nabla_{\mathbf{z}} \tilde{f}(\mathbf{z}(\mathbf{x})) = \nabla_{\mathbf{z}}^{*} \tilde{f}(\mathbf{z}(\mathbf{x}))$
with respect to $\tilde{\mathbf{y}}$ is a nonlinear system, we cannot solve it directly to obtain
$\tilde{\mathbf{y}} = \tilde{\mathcal{M}}(\mathbf{x})$. Moreover, the constraints
$\| \mathcal{M}(\mathbf{x}) - \tilde{\mathcal{M}}(\mathbf{x}) \|_1 \leq \xi$ and $\tilde{\mathbf{y}} \in \Delta^K$
must also be considered.

To find $\tilde{\mathbf{y}}$, GCD employs a solution that (i) returns a valid probability vector $\tilde{\mathbf{y}} \in \Delta^K$,
(ii) keeps $\tilde{\mathbf{y}} = \tilde{\mathcal{M}}(\mathbf{x})$ close to $\hat{\mathbf{y}} = \mathcal{M}(\mathbf{x})$ to satisfy the utility constraint,
and (iii) preserves the predicted label of the target model ($\arg\max_k \hat{\mathbf{y}}_k = \arg\max_k \tilde{\mathbf{y}}_k$).
GCD solves the following minimization problem to obtain $\tilde{\mathbf{y}}$:
\begin{equation}
\label{eq:gc6}
\begin{aligned}
	\underset{\tilde{\mathbf{y}}}{\arg\min} \quad &
	-\left\langle \nabla_{\mathbf{z}} \tilde{f}(\mathbf{z}(\mathbf{x})), \nabla_{\mathbf{z}}^{*} \tilde{f}(\mathbf{z}(\mathbf{x})) \right\rangle \\
	\text{s.t.} \quad & \| \mathcal{M}(\mathbf{x}) - \tilde{\mathcal{M}}(\mathbf{x}) \|_1 \leq \xi .
\end{aligned}
\end{equation}
The inner product in \eqref{eq:gc6} encourages $\nabla_{\mathbf{z}} \tilde{f}(\mathbf{z}(\mathbf{x}))$ to align with
$\nabla_{\mathbf{z}}^{*} \tilde{f}(\mathbf{z}(\mathbf{x}))$ as much as possible, subject to the constraint
$\| \mathcal{M}(\mathbf{x}) - \tilde{\mathcal{M}}(\mathbf{x}) \|_1 \leq \xi$.
After solving the minimization problem in \eqref{eq:gc6}, the resulting $\tilde{\mathbf{y}}$ is normalized by dividing by its sum to ensure that
$\tilde{\mathbf{y}} \in \Delta^K$.
GCD solves the minimization problem in \eqref{eq:gc6} using a trust-region interior-point method \citep{DBLP:books/siam/ConnGT00}, as implemented in the \texttt{scipy} package \citep{virtanen2019scipy}.
In rare cases, the index of the largest component of the resulting $\tilde{\mathbf{y}}$ differs from that of $\hat{\mathbf{y}}$, which reduces the accuracy of the target model. To address this issue, GCD manually adjusts $\tilde{\mathbf{y}}$. Let $i = \arg\max_k \hat{\mathbf{y}}_k$ and $j = \arg\max_k \tilde{\mathbf{y}}_k$. If $i \neq j$, GCD adds $\tilde{\mathbf{y}}_j - \tilde{\mathbf{y}}_i$ to $\tilde{\mathbf{y}}_i$ and subtracts it from $\tilde{\mathbf{y}}_j$. This guarantees that $\arg\max_k \hat{\mathbf{y}}_k = \arg\max_k \tilde{\mathbf{y}}_k$, thereby preserving the accuracy of the target model.

\section{Experiments}
\label{sec:eval_sadeghzadeh}
\subsection{Experimental Setup}

\paragraph{Datasets and Models.}
We evaluate our proposed method, Gradient Confusion Defense across two standard benchmark datasets: CIFAR-10 and ImageNet. To ensure comprehensive assessment, we test on multiple model architectures. For CIFAR-10, we utilize pretrained ResNet-18 models. For ImageNet, we employ Torchvision pretrained ResNet-50 and Vision Transformer (ViT-B/16). All evaluations are performed on $N=1000$ randomly sampled test images per dataset, ensuring statistical significance while maintaining computational efficiency.

\paragraph{Baseline Defenses.}
We compare GCD against five established defense methods implemented in our experimental framework: input Random Noise Defense (iRND)~\citep{Qin2021rnd} applies Gaussian noise with $\sigma=0.01$ to input images. Boundary Defense (oRND)~\citep{Aithal2022ornd} adds Gaussian noise with $\sigma=1.0$ to output logits. Randomized Feature Defense (RFD)~\citep{HungQuang2024rfd} injects noise into intermediate features. Random Logit Scaling (RLS)~\citep{dashtbani2025rls} multiplies logits by a uniformly sampled scale in $[0.5, 10]$. For Adaptive Attack on Attackers (AAA), we implement AAA-sine~\citep{Chen2022aaa} with temperature $1$, deviation $0.5$, attractor interval $4$, calibration weight $5$, learning rate $0.1$, and $100$ inner iterations. We only test AAA-sine, following~\citet{Chen2022aaa} which showed AAA-linear is vulnerable to adaptive attacks.

All hyperparameters were selected based on established literature or thorough validation to ensure a fair comparison. Our defense replaces the model output with the optimized probability vector obtained from Equation~\eqref{eq:gc6}, using an output distortion budget of $\xi = 0.5$. Since the dimension of $\tilde{\mathbf y}$ is large for ImageNet and the trust-region interior-point method is a second-order optimizer, minimizing Equation~\eqref{eq:gc6} incurs a moderate computational cost. Therefore, for ImageNet we select a subset $\mathcal{S}$ of 20 indices of $\tilde{\mathbf y}$ and optimize Equation~\eqref{eq:gc6} only over these components, while fixing the remaining components of $\tilde{\mathbf y}$ to their corresponding values in $\hat{\mathbf y}$. To construct $\mathcal{S}$, we use $\nabla_{\mathbf z}^* \tilde{f}(\mathbf z(\mathbf{x}))$ and select the top-10 and bottom-10 components, which we found to have the largest effect on the objective in Equation~\eqref{eq:gc6}. As a result, for ImageNet the optimized $\tilde{\mathbf y}$ differs from $\hat{\mathbf y}$ in only 20 components (before normalization).



\paragraph{Attacks Configuration.}
We evaluate defenses against four state-of-the-art score-based black-box attacks in untargeted settings: BanditAttack~\citep{Ilyas2019bandit} ($\ell_2$ and $\ell_\infty$ variants), BrusleAttack~\citep{Vo2024brusle} ($\ell_0$ variant), SignAttack~\citep{AlDujaili2020signattack} ($\ell_\infty$ variant), and SquareAttack~\citep{Andriushchenko2020square} ($\ell_2$ and $\ell_\infty$ variants). For CIFAR-10, inputs are scaled to $[0,255]$ with $\ell_\infty$ budget $\epsilon=12.75$ ($0.05$ normalized) and $\ell_2$ budget $\epsilon=255$ ($1.0$ normalized). For ImageNet, inputs are in $[0,1]$ with $\ell_\infty$ budget $\epsilon=0.05$, and $\ell_2$ budgets follow the original implementations (Bandit: $\epsilon=1275$, Square: $\epsilon=5$). All attacks are constrained by a maximum query budget of $Q = 1000$ queries per sample.

\paragraph{Metrics.}
We report three metrics for each (dataset, model, defense, attack) configuration: \textit{Attack Success Rate}, the fraction of samples successfully attacked within the query budget; \textit{Average Queries}, the mean queries per successful attack; and \textit{Clean Accuracy Drop} for each defended model. Together, they capture robustness and utility, enabling evaluation of the defense--utility trade-off.


\begin{table*}[t]
\caption{Black-box attack results on CIFAR-10. ASR and average number of queries are reported.}
\label{tab:bb_attacks_cifar}
\begin{center}
\footnotesize
\setlength{\tabcolsep}{2pt}
\resizebox{\textwidth}{!}{%
\begin{tabular}{@{} l l c c c c c c c @{}}
\toprule
Model & Defense & Acc. Drop
& \multicolumn{6}{c}{Attacks (Success \% / Avg. Queries)} \\
\cmidrule(lr){4-9}
& & & Bandit-$\ell_2$ & Bandit-$\ell_\infty$ & Square-$\ell_2$ & Square-$\ell_\infty$ & Sign-$\ell_\infty$ & BrusLe-$\ell_0$ \\
\midrule

\multirow{6}{*}{ResNet-18} & None & -  & 82.68\% / 545.68 & 98.03\% / 186.16 & 87.76\% / 466.29 & 99.38\% / 112.46 & 98.97\% / 230.05 & 79.78\% / 234.82\\
& iRND & 1\%  & 11.63\% / 333.97 & 41.47\% / 186.49 & 24.92\% / 290.21 & 80.52\% / 178.96 & 66.49\% / 233.48 & 82.9\% / 220.76\\
& oRND & 0.6\%  & 12.53\% / 354.43 & 47.70\% / 158.43 & 20.98\% / 188.31 & 64.09\% / 113.92 & 38.83\% / 170.05 & 68.6\% / 201.10\\
& RLS  & 0\%  & 04.15\% / 345.9 & \textbf{25.11\%} / 23.96 & 12.76\% / 268.69 & 47.31\% / 90.88 & \underline{21.06\%} / 111.41 & \underline{51.77\%} / 316.57\\
& AAA  & 0.1\%  & \underline{03.33\%} / 318 & 38.08\% / 132.26 & \underline{9.77\%} / 174.91 & \underline{43.41\%} / 82.04 & 32.09\% / 177.10 & 66.78\% / 290.17\\
& GCD   & 0\%  & \textbf{0.94\%} / 131.66 & \underline{28.12\%} / 99.37 & \textbf{4.15\%} / 20.95 & \textbf{30.71\%} / 11.10 & \textbf{14.53\%} / 72.14 & \textbf{38.59\%} / 269.82\\
\bottomrule
\end{tabular}
}
\end{center}
\vskip -0.1in
\end{table*}

\begin{table*}[t]
\caption{Black-box attack results on ImageNet. ASR and average number of queries are reported.}
\label{tab:bb_attacks_imagenet}
\begin{center}
\setlength{\tabcolsep}{2pt}
\resizebox{\textwidth}{!}{%
{\footnotesize 
\begin{tabular}{@{} l l c c c c c c @{}}
\toprule
Model & Defense & Acc. Drop
& \multicolumn{5}{c}{Attacks (Success \% / Avg. Queries)} \\
\cmidrule(lr){4-8}
& & & Bandit-$\ell_2$ & Bandit-$\ell_\infty$ & Square-$\ell_2$ & Square-$\ell_\infty$ & Sign-$\ell_\infty$ \\
\midrule

\multirow{6}{*}{ResNet-50} & None & -  & 46.73\% / 761.22 & 87.05\% / 251.79 & 81.45\% / 483.65 & 99.47\% / 133.90 & 91.33\% / 328.71 \\
& iRND & 0.2\%  & 11.02\% / 440.30 & 45.56\% / 141.79 & 26.66\% / 214.17 & 83.49\% / 126.10 & 60.78\% / 229.78 \\
& oRND & 1.7\%  & 51.47\% / 229.94 & 75.95\% / 103.83 & 59.84\% / 208.49 & 83.52\% / 96.91 & 69.47\% / 182.20 \\
& RLS  & 0\%  & \underline{6.68\%} / 618.42 & 47.80\% / 143.72 & 16.29\% / 143.03 & 59.82\% / 104.07 & \underline{21.37\%} / 109.13 \\
& AAA  & 0.3\%  & 11.80\% / 356.96 & \underline{37.94\%} / 85.10 & \underline{15.69\%} / 32.05 & \textbf{47.06\%} / 33.59 & 23.6\% / 109.37 \\
& GCD   & 0\%  & \textbf{3.34\%} / 528.6 & \textbf{35.65\%} / 103.29 & \textbf{13.22\%} / 64.05 & \underline{47.94\%} / 32.17 & \textbf{20.17\%} / 90.15 \\
\midrule

\multirow{6}{*}{ViT-B/16} & None & -  & 38.17\% / 801.41 & 76.46\% / 319.75 & 70.64\% / 566.65 & 97.90\% / 351.15 & 80.05\% / 597.98 \\
& iRND & 0.2\%  & 7.78\% / 471.24 & 25.13\% / 169.19 & 21.77\% / 287.48 & 61.70\% / 252.02 & 38\% / 255.65 \\
& oRND & 1.8\%  & 46.71\% / 266.73 & 69.17\% / 191.07 & 58.88\% / 291.56 & 97.32\% / 84.39 & 81.44\% / 313.67  \\
& RLS  & 0\%  & \underline{4.47\%} / 589.5 & 28.51\% / 224.63 & \textbf{10.29\%} / 148.42 & \underline{30.36\%} / 144.31 & \textbf{13.39\%} / 184.20 \\
& AAA  & 0\%  & 8.17\% / 300.13 & \underline{22.28\%} / 102.86 & \underline{10.9\%} / 43.72 & \textbf{26.74\%} / 92.46 & \underline{15.1\%} / 236.60 \\
& GCD   & 0\%  & \textbf{3.35\%} / 609 & \textbf{21.94\%} / 140.03 & 12.38\% / 455.74 & 33.83\% / 257.21 & 21.44\% / 315.70 \\
\bottomrule
\end{tabular}
}
}
\end{center}
\vskip -0.1in
\end{table*}

\subsection{Results and Analysis}

\paragraph{Overall Performance.}
\cref{tab:bb_attacks_cifar,tab:bb_attacks_imagenet} summarize results on CIFAR-10 (ResNet-18) and ImageNet (ResNet-50, ViT-B/16). Across settings, GCD achieves a strong robustness--utility trade-off, preserving clean accuracy while reducing attack success rates. On CIFAR-10, GCD lowers Square-$\ell_2$ ASR from 87.76\% (undefended) to 4.15\%, and Sign-$\ell_\infty$ from 98.97\% to 14.53\%. This effectiveness transfers to ImageNet: for ResNet-50, GCD incurs no clean accuracy drop while maintaining strong robustness. Results on ViT-B/16 further confirm robustness; for example, Square-$\ell_\infty$ ASR decreases from 97.90\% to 33.83\%.

\paragraph{Comparison with Other Defenses.}
\cref{tab:bb_attacks_cifar,tab:bb_attacks_imagenet} also show that on non-transformer models, GCD outperforms nearly all baselines. While randomized defenses (iRND, oRND, RLS) can reduce attack success, they often incur accuracy drops or weaker protection; iRND and oRND trade accuracy for limited robustness. AAA and RLS preserve accuracy and can be competitive, but are more vulnerable to adaptive attacks (Section \ref{subsec:adaptive}). Table \ref{l1compare} reports $\ell_1$ deviations over 1000 ImageNet samples: GCD exhibits a larger shift, consistent with shaping scores to misalign the induced adversarial loss gradient rather than merely injecting noise.

\begin{table}[t]
  \caption{$\ell_1$ distance between the original output of the model $\hat{y}$ and defense output $\tilde{y}$ for 1000 ImageNet samples.}
  \label{l1compare}
  \centering
  \footnotesize
  \begin{tabular}{lccc}
    \toprule
    Defense & Min $\ell_1$ & Avg $\ell_1$ & Max $\ell_1$ \\
    \midrule
    iRND & 0.00 & 0.11 & 1.03 \\
    oRND & 0.00 & 0.26 & 1.77 \\
    RLS  & 0.00 & 0.33 & 1.74 \\
    AAA  & 0.00 & 0.28 & 1.60 \\
    GCD  & 0.07 & 0.58 & 1.25 \\
    \bottomrule
  \end{tabular}
  \vskip -0.1in
\end{table}

\begin{figure*}[t]
  \vskip 0.2in
  \centering
  \includegraphics[width=\textwidth]{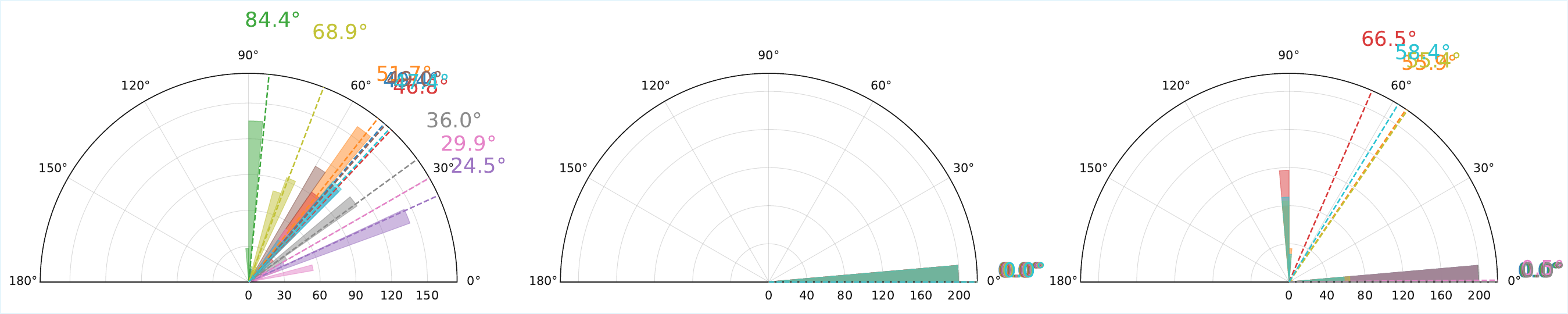}
  \vspace{-0.2em}
  \begin{minipage}{0.32\textwidth}
    \centering
    \small (a) GCD
  \end{minipage}
  \hfill
  \begin{minipage}{0.32\textwidth}
    \centering
    \small (b) AAA
  \end{minipage}
  \hfill
  \begin{minipage}{0.32\textwidth}
    \centering
    \small (c) RLS
  \end{minipage}
  \vspace{0.4em}
  \includegraphics[width=0.6\textwidth]{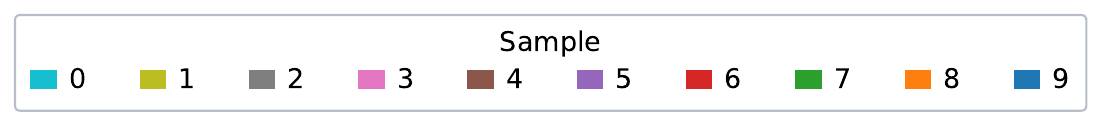}
  \caption{
    \textbf{Polar histograms of gradient direction misalignment.}
For 10 randomly selected samples from the ImageNet test set, we plot the angular deviation
$\theta=\angle(\nabla_{\mathbf x} f(\mathbf x), \nabla_{\mathbf x} \tilde{f}(\mathbf x))$
between the direction $\nabla_{\mathbf x} \tilde{f}(\mathbf x)$ induced by the defended score responses
(i.e., the direction the attacker would follow) and the true loss gradient
$\nabla_{\mathbf x} f(\mathbf x)$.
Each color corresponds to 200 square-$\ell_\infty$ attack iterations (200 queries used to generate adversarial example for a sample).
Dashed rays and numeric labels indicate the mean deviation for each sample.
}

  \label{fig:histograms_abopt_atv}
\end{figure*}

\paragraph{Gradient Misalignment Evidence.}
Figure~\ref{fig:histograms_abopt_atv} shows $\theta=\angle(\nabla_{\mathbf x} f(\mathbf x), \nabla_{\mathbf x} \tilde{f}(\mathbf x))$, the deviation between the true loss gradient $\nabla_{\mathbf x} f(\mathbf x)$ and the gradient $\nabla_{\mathbf x} \tilde{f}(\mathbf x)$ implied by the defended score responses (i.e., the direction the attacker follows). GCD induces a \emph{systematic} bias in this recovered direction, increasing the query cost of score-based attacks. In particular, GCD yields large deviations across queries (e.g., $24.5^\circ$--$84.4^\circ$, many near $\sim 50^\circ$), whereas AAA concentrates at $\theta\approx 0^\circ$ (nearly revealing the true direction). RLS shows mixed behavior, including near-zero angles that still provide a reliable descent direction and angles near $90^\circ$. We hypothesize this occurs because the RLS transform maps some target-model output probability vectors to one-hot vectors, making $\nabla_{\mathbf x} \tilde{f}(\mathbf x) \approx 0$. As a result, GCD’s persistent misalignment compounds over iterations, increasing query complexity and reducing attack success, while defenses that often yield $\theta\approx 0^\circ$ preserve a low-query path. These observations support our theoretical analysis.


Figure \ref{fig:trajectory} demonstrates trajectory-level analysis for 10 random CIFAR‑10 samples over 100 Square Attack iterations: for each query iteration $t$ along the attack trajectory, we measure the angle $\theta_t$ between the direction exposed under the defense and the corresponding undefended direction. This is the operational quantity that matters to an iterative score-based attacker. Based on these results, AAA does exhibit occasional local reversals, but they are not persistent: after a brief transient, its mean angle quickly settles to a small range (about $8^\circ$--$10^\circ$ in our trace plot), and many sample traces remain near low angles. In contrast, GCD maintains substantially larger misalignment throughout the trajectory, with the mean staying around $25^\circ$--$28^\circ$, and many samples remaining in the $30^\circ$--$45^\circ$ range over tens of queries. So the practical difference is not whether a defense can reverse a local slope at isolated points, but whether it sustains input-space misalignment along the whole optimization path. GCD does; AAA usually does not.
This also explains the mechanism. AAA heuristically flips local slopes in output space, but it does not optimize the induced input-space direction after composition with the model Jacobian. GCD is Jacobian-aware: Eqs. \eqref{eq:mainproblem}, \eqref{eq:optimal_grad}, and \eqref{eq:gc6} explicitly optimize the exposed direction in input space. That is precisely the quantity that matters for score-based attacks. This interpretation is also consistent with our theory: Theorem \ref{thm:rs_switch_failure} shows that larger angular misalignment increases the probability of accept/reject errors in search-based attacks, and when that larger angle is maintained across many iterations, those errors compound over the trajectory.

\begin{figure}[t]
    \centering
    \includegraphics[width=0.75\columnwidth]{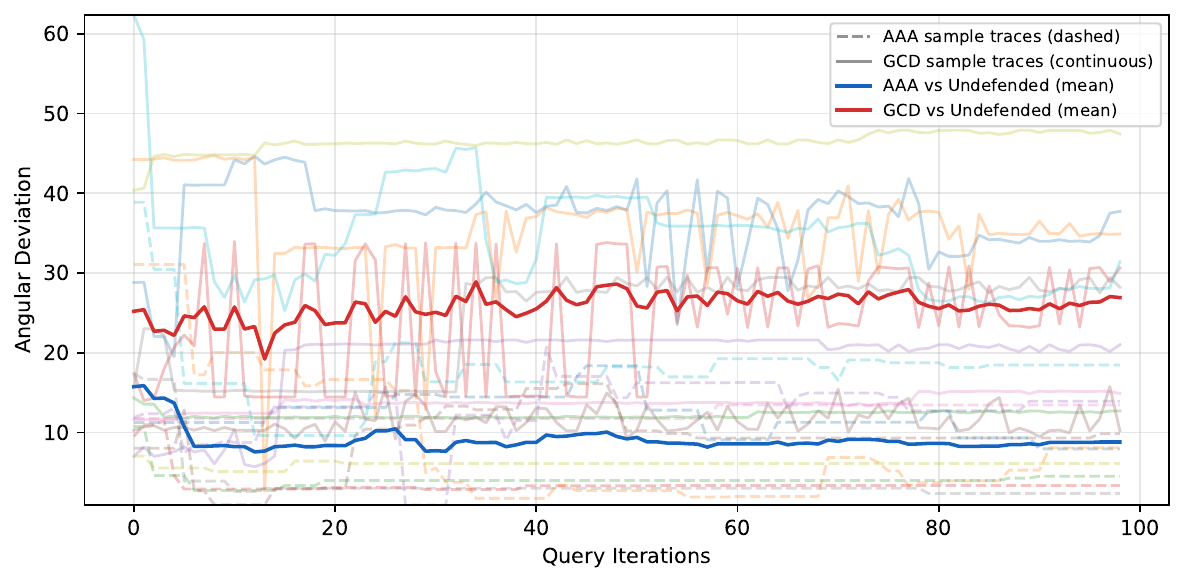}
    \caption{Angular deviation between defense-exposed and undefended gradient directions over 100 Square Attack iterations on 10 CIFAR-10 samples.}
    \label{fig:trajectory}
\end{figure}

\begin{table}[t]
  \caption{Adaptive Attacks on GCD, RLS, and AAA. We incorporate two adaptive techniques—Expectation over Transformation (EOT) and Pendulum—into the Square Attack framework. For both adaptive variants, we set $k=10$ and evaluate performance on $1000$ ImageNet samples using ViT-B/16.}
  \label{adaptive}
  \begin{center}
    \begin{small}
      \begin{sc}
      \resizebox{0.55\columnwidth}{!}{%
        \begin{tabular}{lccc}
          \toprule
          Defense  & EOT & Pen-Square & Max(EOT,Pen)  \\
          \midrule
          AAA & 21.94\% / 85.44 & 40.53\% / 306.08 & 40.53\% \\
          RLS & 36.19\% / 215.83 & 23.55\% / 85.44 & \underline{36.19\%} \\
          GCD & 24.04\% / 159.69 & 30.12\% / 209.33 & \textbf{30.12\%} \\
          \bottomrule
        \end{tabular}
        }
      \end{sc}
    \end{small}
    \end{center}
  \vskip -0.1in
\end{table}

\subsection{Adaptive Attack Evaluation}
\label{subsec:adaptive}
To assess defense resilience against adaptive adversaries, we evaluate GCD alongside AAA and RLS against two adaptive attacks: Expectation over Transformation (EOT) \citep{Athalye2018eot} and PendulumAttack \citep{dashtbani2025rls}.
For sample $\mathbf x$, EOT counteracts random-based defenses by averaging over \(k\) independently randomized defense outputs $(\frac{1}{k} \sum^{i=0}_{k} \tilde{\mathcal{M}}^{(i)})$ to obtain an unbiased estimatimation of the $\hat{\mathbf y} = \mathcal{M}(\mathbf x)$. PendulumAttack employs iterative estimation and exploitation of defense patterns.

Table~\ref{adaptive} reveals that GCD maintains performance even under adaptive settings. The primary result is reported in the Max column, which corresponds to a worst-case adaptive threat model in which the attacker uses the more successful of the two evaluated adaptive variants. Under this setting, GCD exhibits stronger robustness than both AAA and RLS.

Figure~\ref{fig:trajectory} provides a mechanism-level interpretation of this result. GCD introduces a deterministic bias into the exposed objective, whereas randomization-based defenses primarily introduce variance and AAA primarily creates local slope perturbations. Adaptive averaging is designed to reduce variance; it does not remove a deterministic directional bias. For a deterministic defense, the relevant question is therefore whether the exposed direction remains persistently displaced along the attack trajectory. Figure~\ref{fig:trajectory} shows that GCD maintains this displacement, while AAA realigns much more quickly.

This evidence helps explain the observed adaptive comparison, although it does not by itself establish causality. Adaptive attacks may smooth local irregularities or exploit predictable patterns, but they still optimize the systematically altered objective exposed by GCD. By contrast, AAA transforms the loss into a structured oscillatory surface. Pendulum-style attacks can exploit such structure by reversing the optimization direction when progress stalls, thereby escaping local minima and maxima. The trajectory plot therefore complements Table~\ref{adaptive}, it indicates that GCD's effect is sustained across successive queries rather than being limited to isolated local slope changes.





\section{Discussion and Conclusion}
Our comprehensive evaluation across two datasets, multiple model architectures, four attack methods, and two threat models demonstrates that Gradient Confusion represents an effective new paradigm for black-box defense. By explicitly shaping the gradient information revealed to attackers, GCD Consistently lower attack success rates across all evaluated configurations. Besides, we do not degrade the clean accuracy. Also, GCD Maintains protection even against sophisticated adaptive attacks.
These results establish GCD as a state-of-the-art defense against score-based black-box attacks, offering a principled alternative to randomization-based approaches. The method's model-agnostic implementation and deterministic operation make it particularly suitable for deployment in security-critical applications where both robustness and reliable performance are paramount.
However, one limitation warrants discussion. GCD requires Jacobian computations, adding computational overhead (approximately 15\% inference time increase). Future work could explore alternative optimization methods and workarounds to address this obstacle.


\bibliographystyle{cas-model2-names}
\bibliography{example_paper}
\clearpage
\appendix
\section{Notations}
Here one can find all notations and symbols used in the paper.

\begin{table}[H]
    \centering
    \caption{Summary of Notations and Symbols used in the Gradient Confusion Defense (GCD) paper.}
    \label{tab:notations}
    \renewcommand{\arraystretch}{1.3}
    \small
    \begin{tabular}{l p{10cm}}
        \toprule
        \textbf{Symbol} & \textbf{Description} \\
        \midrule
        \multicolumn{2}{c}{\textit{General Model and Inputs}} \\
        \midrule
        $\mathcal{X}$ & Input space, $\mathcal{X} \subseteq [0,1]^d$ \\
        $d$ & Dimension of the input space \\
        $[K]$ & Label set $\{1, \dots, K\}$ \\
        $\mathcal{M}$ & Classifier mapping $\mathcal{X} \rightarrow \Delta^K$ \\
        $\mathbf{z}(\mathbf{x})$ & Logits vector output by the model for input $\mathbf{x}$ \\
        $\sigma(\cdot)$ & Softmax function \\
        $\Delta^K$ & Probability simplex $\{\mathbf{p} \in \mathbb{R}^K \mid \sum_i \mathbf{p}_i = 1\}$ \\
        $\hat{\mathbf{y}}, \mathcal{M}(\mathbf{x})$ & Original probability vector output by the target model \\
        $k$ & True label of input $\mathbf{x}$ \\
        $J(\mathbf{x})$ & Jacobian of the model's logits with respect to the input, $\nabla_\mathbf{x} \mathbf{z}(\mathbf{x})$ \\
        \midrule
        \multicolumn{2}{c}{\textit{Attacks and Optimization}} \\
        \midrule
        $\mathbf{x}_{adv}$ & Adversarial example ($\mathbf{x}_{adv} = \mathbf{x} + \boldsymbol{\delta}$) \\
        $R$ & Perturbation budget (radius of the allowed perturbation ball) \\
        $Q$ & Maximum query budget \\
        $f(\mathbf{x})$ & Adversarial loss function based on the original model output \\
        $i, j$ & Indices of the largest and second-largest probability components \\
        $\hat{\nabla}_\mathbf{x} f_\mu(\mathbf{x})$ & Estimated gradient of the smoothed function \\
        $\mathbf{u}, \mathbf{u}_t$ & Random search direction vector, $\mathbf{u} \sim \mathcal{N}(\mathbf{0}, \mathbf{I})$ \\
        $\mu$ & Smoothing parameter (step size for finite differences) \\
        $f_\mu(\mathbf{x})$ & Gaussian-smoothed version of the loss function \\
        $\eta_t$ & Step size (learning rate) at iteration $t$ \\
        $\Pi_\xi$ & Projection operator onto the $l_p$-ball of radius $\xi$ \\
        $h_t$ & One-step improvement signal in random search \\
        \midrule
        \multicolumn{2}{c}{\textit{Gradient Confusion Defense (GCD)}} \\
        \midrule
        $\tilde{\mathbf{y}}, \tilde{\mathcal{M}}(\mathbf{x})$ & Modified/Defended probability vector returned to the user \\
        $\tilde{\mathbf{z}}(\mathbf{x})$ & Perturbed logits corresponding to $\tilde{\mathbf{y}}$ \\
        $\tilde{f}(\mathbf{x})$ & Adversarial loss function based on the modified output $\tilde{\mathbf{y}}$ \\
        $\xi$ & Distortion budget (constraint on $l_1$ distance between $\hat{\mathbf{y}}$ and $\tilde{\mathbf{y}}$) \\
        $f_{\text{def}}(\mathbf{x})$ & Randomized defended objective function (theoretical model) \\
        $\lambda$ & Bernoulli random variable used in the theoretical switching defense \\
        $p$ & Probability parameter for the Bernoulli distribution ($\lambda \sim \text{Bernoulli}(p)$) \\
        \midrule
        \multicolumn{2}{c}{\textit{Theoretical Analysis}} \\
        \midrule
        $L_0(f)$ & Lipschitz continuity constant of function $f$ \\
        $L_1(f)$ & Lipschitz constant of the gradient (smoothness constant) of $f$ \\
        $T$ & Total number of optimization steps/queries \\
        $\epsilon$ & Small constants bounding the difference between functions ($|f - \tilde{f}|$) \\
        $\epsilon'$ & Small constants bounding the difference between functions ($|f_\mu - f|$) \\
        $h_{\text{def}}$ & Improvement signal computed from the defender's returned values \\
        \bottomrule
    \end{tabular}
\end{table}

\section{Theoretical Analysis of Gradient Confusion for Gradient Estimation-Based Attacks}
\label{apx:conv-theorem-proof}

In this appendix, we provide the theoretical analysis of Theorem~\ref{thm:prob-converge}.
We show that by incorporating a second function $\tilde{f}$, which is close to the original function $f$ but has misaligned gradients, we can increase the convergence bound of gradient-estimation-based attacks.
This analysis provides insights into how optimizing a confusion term can enhance the robustness of models against such attacks.
We first state preliminaries and lemmas used in the proof, and then present the detailed proof of Theorem~\ref{thm:prob-converge}.

\subsection{Preliminaries}
\begin{definition}[$\mu$-Gaussian Smoothing]
For a function $f:\mathbb{R}^d \to \mathbb{R}$, we define its $\mu$-Gaussian smoothed version as:
\begin{equation}
    f_\mu(\mathbf{x}) = \mathbb{E}_{\mathbf{u} \sim \mathcal{N}(\mathbf{0}, \mathbf{I_d})}[f(\mathbf{x} + \mu \mathbf{u})]
\end{equation}
\end{definition}

Now, we define the gradient estimator as
\begin{equation}
g_\mu(\mathbf{x}) = \frac{f(\mathbf{x}+\mu \mathbf{u}) - f(\mathbf{x})}{\mu}\,\mathbf{u},
\end{equation}
where $\mathbf{u} \sim \mathcal{N}(\mathbf{0}, \mathbf{I}_d)$.
Using the change of variables $\mathbf{y} := \mathbf{x} + \mu \mathbf{u}$, we can show that the gradient of the smoothed function can be written as
\begin{equation}
    \begin{aligned}
        \nabla f_\mu(\mathbf{x})
        &= \nabla \left(\frac{1}{(2\pi)^{d/2}} \int f(\mathbf{x}+\mu \mathbf{u}) e^{-\frac{1}{2}\|\mathbf{u}\|^2} d\mathbf{u} \right) \\
        &= \nabla \left(\frac{1}{\mu^{d+1}(2\pi)^{d/2}} \int f(\mathbf{y}) e^{-\frac{1}{2\mu^2}\|\mathbf{y}-\mathbf{x}\|^2} d \mathbf{y} \right) & \qquad (\text{Change of variables})\\
        &= \frac{1}{\mu^{d+1}(2\pi)^{d/2}} \int f(\mathbf{y}) \left(-\frac{\mathbf{y}-\mathbf{x}}{\mu^2}\right) e^{-\frac{1}{2\mu^2}\|\mathbf{y}-\mathbf{x}\|^2} d \mathbf{y} \\
        &= \frac{1}{\mu^{d+1}(2\pi)^{d/2}} \int f(\mathbf{x}+\mu \mathbf{u}) \left(-\frac{\mu \mathbf{u}}{\mu^2}\right) e^{-\frac{1}{2}\|\mathbf{u}\|^2} \mu^d d \mathbf{u} & \qquad (\text{Change of variables})\\
        &= \frac{1}{(2\pi)^{d/2}} \int f(\mathbf{x}+\mu \mathbf{u}) \left(-\frac{\mathbf{u}}{\mu}\right) e^{-\frac{1}{2}\|\mathbf{u}\|^2} d \mathbf{u} \\
        &= \mathbb{E}_{\mathbf{u} \sim \mathcal{N}(\mathbf{0}, \mathbf{I_d})}\left[\frac{f(\mathbf{x}+\mu \mathbf{u})}{\mu} \mathbf{u}\right] & \qquad \left(\mathbb{E}_{\mathbf{u} \sim \mathcal{N}(\mathbf{0}, \mathbf{I_d)}}\left[\frac{f(\mathbf{x})}{\mu} \mathbf{u}\right] = 0\right)\\
        &= \mathbb{E}_{\mathbf{u} \sim \mathcal{N}(\mathbf{0}, \mathbf{I_d)}}\left[\frac{f(\mathbf{x}+\mu \mathbf{u}) - f(\mathbf{x})}{\mu} \mathbf{u}\right] \\
        &= \mathbb{E}_{\mathbf{u} \sim \mathcal{N}(\mathbf{0}, \mathbf{I_d})}[g_\mu(\mathbf{x})]
    \end{aligned}
\end{equation}
Thus, we have the following lemma:
\begin{lemma}
\label{lem:grad-estimator}
The gradient estimator $g_\mu(\mathbf{x})$ is an unbiased estimator of the gradient of the smoothed function $f_\mu(\mathbf{x})$, i.e., $\mathbb{E}_{\mathbf{u}}\!\left[g_\mu(\mathbf{x})\right] = \nabla f_\mu(\mathbf{x})$.
\end{lemma}

We also use the following lemma on the smoothness of the Gaussian-smoothed function from \citet{nesterov2017random}:
\begin{lemma}
\label{lem:smoothness}
If $f$ is $L_0(f)$-Lipschitz continuous, then its smoothed version $f_\mu$ is differentiable and has a Lipschitz continuous gradient with constant
\[
L_1(f_\mu) = \frac{2\sqrt{d}}{\mu} L_0(f).
\]
\end{lemma}

We also use the following lemma on the moments of Gaussian random variables, which follows from standard concentration properties of Gaussian distributions:
\begin{lemma}
\label{lem:gaussian-moments}
For $\mathbf{u} \sim \mathcal{N}(\mathbf{0}, \mathbf{I}_d)$, we have
\begin{equation}
\mathbb{E}\!\left[\|\mathbf{u}\|^2\right] \le d,
\end{equation}
and
\begin{equation}
\mathbb{E}\!\left[\|\mathbf{u}\|^4\right] \le (d+4)^2.
\end{equation}
\end{lemma}

\subsection{Bias and variance of gradient estimator}
Let $f:\mathbb{R}^d \to \mathbb{R}$ and $\tilde{f}:\mathbb{R}^d \to \mathbb{R}$ be two functions such that, for all $\mathbf{x} \in \mathbb{R}^d$, $|f(\mathbf{x}) - \tilde{f}(\mathbf{x})| \leq \epsilon$.
We define a randomized objective $h$ as
\begin{equation}
h(\mathbf{x}) = \lambda f(\mathbf{x}) + (1-\lambda)\tilde{f}(\mathbf{x}),
\end{equation}
where $\lambda \sim \mathrm{Bernoulli}(p)$.
Thus, the corresponding gradient estimator is given by
\begin{equation}
    \begin{aligned}
        g_\mu(\mathbf{x}) &= \frac{\left(\lambda_1f(\mathbf{x}+\mu \mathbf{u})+(1-\lambda_1)\tilde{f}(\mathbf{x}+\mu \mathbf{u})\right) - \left(\lambda_2f(\mathbf{x})+(1-\lambda_2)\tilde{f}(\mathbf{x})\right)}{\mu} \mathbf{u} \\
        &= \left[\lambda_1\frac{f(\mathbf{x}+\mu \mathbf{u}) - f(\mathbf{x})}{\mu} + (1-\lambda_1)\frac{\tilde{f}(\mathbf{x}+\mu \mathbf{u}) - \tilde{f}(\mathbf{x})}{\mu} + \frac{(\lambda_1-\lambda_2)(f(\mathbf{x}) - \tilde{f}(\mathbf{x}))}{\mu}\right] \mathbf{u}
    \end{aligned}
\end{equation}
Thus, for the expectation of the gradient estimator, we have:
\begin{equation}
    \begin{aligned}
        \mathbb{E}_{\lambda_1, \lambda_2}[g_\mu(\mathbf{x})] &= p \left(\frac{f(\mathbf{x}+\mu \mathbf{u}) - f(\mathbf{x})}{\mu} \mathbf{u}\right) + (1-p) \left(\frac{\tilde{f}(\mathbf{x}+\mu \mathbf{u}) - \tilde{f}(\mathbf{x})}{\mu} \mathbf{u} \right) \\
    \end{aligned}
\end{equation}
From lemma \ref{lem:grad-estimator} we have:
\begin{equation}
    \label{eq:exp_bias}
    \mathbb{E}_{\mathbf{u}, \lambda_1, \lambda_2}[g_\mu(\mathbf{x})] = p \nabla f_\mu(\mathbf{x}) + (1-p) \nabla \tilde{f}_\mu(\mathbf{x})
\end{equation}
Now we need to compute the variance of the gradient estimator:
\begin{equation}
    \label{eq:variance}
    \begin{gathered}
        \|g_\mu(\mathbf{x})\|^2 = \underbrace{\left[\lambda_1\frac{f(\mathbf{x}+\mu \mathbf{u}) - f(\mathbf{x})}{\mu} + (1-\lambda_1)\frac{\tilde{f}(\mathbf{x}+\mu \mathbf{u}) - \tilde{f}(\mathbf{x})}{\mu} + \frac{(\lambda_1-\lambda_2)(f(\mathbf{x}) - \tilde{f}(\mathbf{x}))}{\mu}\right]^2}_{A}\|\mathbf{u}\|^2 \\
    \end{gathered}
\end{equation}
Now set $\alpha = \frac{f(\mathbf{x}+\mu \mathbf{u}) - f(\mathbf{x})}{\mu}$, $\tilde{\alpha} = \frac{\tilde{f}(\mathbf{x}+\mu \mathbf{u}) - \tilde{f}(\mathbf{x})}{\mu}$, and $\delta = \frac{f(\mathbf{x}) - \tilde{f}(\mathbf{x})}{\mu}$. Thus for $A$ we have:
\begin{equation}
    \begin{gathered}
        A = \lambda_1^2 \alpha^2 + (1-\lambda_1)^2 \tilde{\alpha}^2 + (\lambda_1-\lambda_2)^2 \delta^2 + 2\lambda_1(1-\lambda_1)\alpha\tilde{\alpha} + 2\lambda_1(\lambda_1-\lambda_2)\alpha\delta + 2(1-\lambda_1)(\lambda_1-\lambda_2)\tilde{\alpha}\delta \\
        \implies \mathbb{E}_{\lambda_1, \lambda_2}[A] = p\alpha^2 + (1-p)\tilde{\alpha}^2 + 2p(1-p)(\alpha - \tilde{\alpha})\delta+ 2p(1-p)\delta^2 \\
        \hfil = p\alpha^2 + (1-p)\tilde{\alpha}^2 + 2p(1-p)(\alpha - \tilde{\alpha} + \delta)\delta \\[1em]
        \mathbb{E}_{\lambda_1, \lambda_2}[A] = p\left(\frac{f(\mathbf{x}+\mu \mathbf{u}) - f(\mathbf{x})}{\mu}\right)^2 + (1-p)\left(\frac{\tilde{f}(\mathbf{x}+\mu \mathbf{u}) - \tilde{f}(\mathbf{x})}{\mu}\right)^2 \\
        + \frac{2p(1-p)}{\mu^2}\left(f(\mathbf{x}+\mu \mathbf{u}) - \tilde{f}(\mathbf{x}+\mu \mathbf{u})\right)\left(f(\mathbf{x}) - \tilde{f}(\mathbf{x})\right) \\
    \end{gathered}
\end{equation}
Now using the bound $|f(\mathbf{x}) - \tilde{f}(\mathbf{x})| \leq \epsilon$, and Lipschitz continuity of $f$ and $\tilde{f}$, we have:
\begin{equation}
    \begin{gathered}
        \mathbb{E}_{\lambda_1, \lambda_2}[A] \leq pL^2_0(f)\|\mathbf{u}\|^2 + (1-p)L^2_0(\tilde{f})\|\mathbf{u}\|^2 + \frac{2p(1-p)\epsilon^2}{\mu^2}\\
    \end{gathered}
\end{equation}
Putting this into equation \ref{eq:variance}, we have:
\begin{equation}
    \mathbb{E}_{\lambda_1, \lambda_2}[\|g_\mu(\mathbf{x})\|^2] \le pL^2_0(f)\|\mathbf{u}\|^4 + (1-p)L^2_0(\tilde{f})\|\mathbf{u}\|^4 + \frac{2p(1-p)\epsilon^2}{\mu^2}\|\mathbf{u}\|^2
\end{equation}
Now taking expectation over $\mathbf{u}$, using lemma \ref{lem:gaussian-moments} we have:
\begin{equation}
    \label{eq:exp_variance}
    \mathbb{E}_{\mathbf{u}, \lambda_1, \lambda_2}[\|g_\mu(\mathbf{x})\|^2] \leq pL^2_0(f)(d+4)^2 + (1-p)L^2_0(\tilde{f})(d+4)^2 + \frac{2p(1-p)\epsilon^2}{\mu^2}d
\end{equation}

\subsection{Descent Lemma}
We now define the optimization process.
Starting from $\mathbf{x}_0$, at each iteration $t$ we sample
$\lambda_1^{(t)}, \lambda_2^{(t)} \sim \mathrm{Bernoulli}(p)$ and
$\mathbf{u}^{(t)} \sim \mathcal{N}(\mathbf{0}, \mathbf{I}_d)$,
and compute the gradient estimator $g_\mu(\mathbf{x}_t)$.
We then update the iterate as
\[
\mathbf{x}_{t+1} = \mathbf{x}_t - \eta_t g_\mu(\mathbf{x}_t),
\]
where $\eta_t$ denotes the step size at iteration $t$.
Consequently, each iterate $\mathbf{x}_t$ is a random variable that depends on all previous samples
$\{\lambda_1^{(i)}, \lambda_2^{(i)}, \mathbf{u}^{(i)}\}_{i < t}$.
We define the history up to iteration $t$ as
\[
\mathcal{Z}_t = \left\{\lambda_1^{(i)}, \lambda_2^{(i)}, \mathbf{u}^{(i)}\right\}_{i < t}.
\]
Our goal is to analyze the convergence of this process in expectation over all sources of randomness.
We assume the process is run for $T$ iterations.

Having access to the variance and bias of the gradient estimator, we can now use the standard descent lemma for zero order optimization.
\begin{equation}
f_\mu(\mathbf{x}_{t+1}) \leq f_\mu(\mathbf{x}_t) - \eta_t \langle \nabla f_\mu(\mathbf{x}_t), g_\mu(\mathbf{x}_t) \rangle + \frac{L_1(f_\mu)}{2} \eta_t^2 \|g_\mu(\mathbf{x}_t)\|^2
\end{equation}
Using equations \ref{eq:exp_bias} and \ref{eq:exp_variance}, and taking expectation over $\lambda_1^{(t)}, \lambda_2^{(t)}$ and $\mathbf{u}^{(t)}$ we have:
\begin{equation}
    \begin{gathered}
        \mathbb{E}_{\mathbf{u}^{(t)}, \lambda_1^{(t)}, \lambda_2^{(t)}}[f_\mu(\mathbf{x}_{t+1})] \leq f_\mu(\mathbf{x}_t) - \eta_t p \|\nabla f_\mu(\mathbf{x}_t)\|^2 - \eta_t (1-p) \langle \nabla f_\mu(\mathbf{x}_t), \nabla \tilde{f}_\mu(\mathbf{x}_t) \rangle \\[.6em]
        + \frac{L_1(f_\mu)}{2} \eta_t^2 \left(pL^2_0(f)(d+4)^2 + (1-p)L^2_0(\tilde{f})(d+4)^2 + \frac{2p(1-p)\epsilon^2}{\mu^2}d\right)
    \end{gathered}
\end{equation}
Now taking expectation over the history $\mathcal{Z}_t$, we have:
\begin{equation}
    \begin{gathered}
        \mathbb{E}_{\mathcal{Z}_t, \mathbf{u}^{(t)}, \lambda_1^{(t)}, \lambda_2^{(t)}}[f_\mu(\mathbf{x}_{t+1})] = \mathbb{E}_{\mathcal{Z}_{t+1}}[f_\mu(\mathbf{x}_{t+1})] \\ \leq \mathbb{E}_{\mathcal{Z}_t}[f_\mu(\mathbf{x}_t)] - \eta_t p \mathbb{E}_{\mathcal{Z}_t}[\|\nabla f_\mu(\mathbf{x}_t)\|^2] - \eta_t (1-p) \mathbb{E}_{\mathcal{Z}_t}[\langle \nabla f_\mu(\mathbf{x}_t), \nabla \tilde{f}_\mu(\mathbf{x}_t) \rangle] \\[.6em]
        + \frac{L_1(f_\mu)}{2} \eta_t^2 \left(pL^2_0(f)(d+4)^2 + (1-p)L^2_0(\tilde{f})(d+4)^2 + \frac{2p(1-p)\epsilon^2}{\mu^2}d\right)
    \end{gathered}
\end{equation}
Let's define new notations for simplicity:
\begin{equation}
    \begin{cases}
        A_t = \mathbb{E}_{\mathcal{Z}_t}[f_\mu(\mathbf{x}_t)] \\
        B_t = \mathbb{E}_{\mathcal{Z}_t}[\|\nabla f_\mu(\mathbf{x}_t)\|^2] \\
        C_t = \mathbb{E}_{\mathcal{Z}_t}[\langle \nabla f_\mu(\mathbf{x}_t), \nabla \tilde{f}_\mu(\mathbf{x}_t) \rangle] \\
        D = pL^2_0(f)(d+4)^2 + (1-p)L^2_0(\tilde{f})(d+4)^2 + \frac{2p(1-p)\epsilon^2}{\mu^2}d \\
    \end{cases}
\end{equation}
Rewrite the inequality with new notations:
\begin{equation}
    A_{t+1} \leq A_t - \eta_t p B_t - \eta_t (1-p) C_t + \frac{L_1(f_\mu)}{2} \eta_t^2 D
\end{equation}
Rearranging the terms, we have:
\begin{equation}
    A_{t+1} - A_t \leq - \eta_t p B_t - \eta_t (1-p) C_t + \frac{L_1(f_\mu)}{2} \eta_t^2 D
\end{equation}
Thus for different iterations $t = 0, \ldots, T-1$, we have:
\begin{equation}
    \begin{gathered}
        A_{1} - A_0 \leq - \eta_0 p B_0 - \eta_0 (1-p) C_0 + \frac{L_1(f_\mu)}{2} \eta_0^2 D \\
        A_{2} - A_1 \leq - \eta_1 p B_1 - \eta_1 (1-p) C_1 + \frac{L_1(f_\mu)}{2} \eta_1^2 D \\
        \vdots \\
        A_{T} - A_{T-1} \leq - \eta_{T-1} p B_{T-1} - \eta_{T-1} (1-p) C_{T-1} + \frac{L_1(f_\mu)}{2} \eta_{T-1}^2 D \\
    \end{gathered}
\end{equation}
Summing all these inequalities, we have:
\begin{equation}
    A_T - A_0 \leq - p \sum_{t=0}^{T-1} \eta_t B_t - (1-p) \sum_{t=0}^{T-1} \eta_t C_t + \frac{L_1(f_\mu)}{2} D \sum_{t=0}^{T-1} \eta_t^2
\end{equation}
Here $A_0 = f_\mu(\mathbf{x}_0)$ and $A_T = \mathbb{E}_{\mathcal{Z}_T} [f_\mu(\mathbf{x}_T)]$ is the expectation of the function value at the end of the optimization process. Let's assume we use constant step size $\eta_t = \eta$ for all $t$. Thus we have:
\begin{equation}
    A_T - A_0 \leq - \eta p \sum_{t=0}^{T-1} B_t - \eta (1-p) \sum_{t=0}^{T-1} C_t + \frac{L_1(f_\mu)}{2} \eta^2 D T
\end{equation}
Taking average over $T$ iterations, we have:
\begin{equation}
    \frac{1}{\eta T} (A_T - A_0) \leq - p \frac{1}{T} \sum_{t=0}^{T-1} B_t - (1-p) \frac{1}{T} \sum_{t=0}^{T-1} C_t + \frac{L_1(f_\mu)}{2} \eta D
\end{equation}
Rearranging the terms, we have:
\begin{equation}
    p \frac{1}{T} \sum_{t=0}^{T-1} B_t + (1-p) \frac{1}{T} \sum_{t=0}^{T-1} C_t \leq \frac{A_0 - A_T}{\eta T} + \frac{L_1(f_\mu)}{2} \eta D
\end{equation}
Now minimize the right hand side with respect to $\eta$:
\begin{equation}
    \eta^\ast = \sqrt{\frac{2(A_0 - A_T)}{L_1(f_\mu) D T}}
\end{equation}
Thus we have:
\begin{equation}
    p \frac{1}{T} \sum_{t=0}^{T-1} B_t + (1-p) \frac{1}{T} \sum_{t=0}^{T-1} C_t \leq \sqrt{2L_1(f_\mu) D \left(\frac{A_0 - A_T}{T}\right)}
\end{equation}
From Lipschitz continuity of $f$ and $\tilde{f}$, we know that $f_\mu$ and $\tilde{f}_\mu$ are also Lipschitz continuous with the same constants. Thus we can bound $A_0 - A_T$ as:
\begin{equation}
    |A_0 - A_T| \le L_0(f) \|\mathbf{x}_0 - \mathbf{x}^\ast\| \le L_0(f) R
\end{equation}
Where $R$ is the radius of the ball for optimization budget.
Thus we have the final convergence bound as:
\begin{equation}
    p \frac{1}{T} \sum_{t=0}^{T-1} B_t + (1-p) \frac{1}{T} \sum_{t=0}^{T-1} C_t \leq \sqrt{2L_1(f_\mu) D \left(\frac{L_0(f) R}{T}\right)}   
\end{equation}
From the definition of $f_\mu$ and lemma ~\ref{lem:smoothness} we have $L_1(f_\mu) = \frac{2d^{1/2}}{\mu} L_0(f)$, therefore we have:
\begin{equation}
    p \frac{1}{T} \sum_{t=0}^{T-1} B_t + (1-p) \frac{1}{T} \sum_{t=0}^{T-1} C_t \leq \sqrt{\frac{4d^{1/2} L_0(f)^2 R}{\mu T} D}   
\end{equation}
We also need to choose $\mu$ small enough to have $|f_\mu(\mathbf{x}) - f(\mathbf{x})| \le \epsilon^\prime$. To have this bound, we choose $\mu \le \frac{\epsilon^\prime}{d^{1/2}L_0(f)}$ as in \citet{nesterov2017random}. Thus we can still simplify the bound as:
\begin{equation}
    p \frac{1}{T} \sum_{t=0}^{T-1} B_t + (1-p) \frac{1}{T} \sum_{t=0}^{T-1} C_t \leq \sqrt{\frac{4d L_0(f)^3 R}{\epsilon^\prime T} D}
\end{equation}
Let's also plug this value in $D$ and rewrite that:
\begin{equation}
    D = pL^2_0(f)(d+4)^2 + (1-p)L^2_0(\tilde{f})(d+4)^2 + \frac{2p(1-p)\epsilon^2 d^2 L_0(f)^2}{\epsilon^{\prime 2}}
\end{equation}
The input dimension $d$ usually dominates the constant $4$ in $(d+4)$, thus we can simplify the final bound as:
\begin{equation}
    \begin{aligned}
        D &\approx pL^2_0(f)d^2 + (1-p)L^2_0(\tilde{f})d^2 + \frac{2p(1-p)\epsilon^2 d^2 L_0(f)^2}{\epsilon^{\prime 2}}\\
        &= pd^2L_0(f)^2 \left(1 + \frac{(1-p)}{p}\frac{L^2_0(\tilde{f})}{L^2_0(f)} + 2(1-p)\frac{\epsilon^2}{\epsilon^{\prime 2}}\right)
    \end{aligned}
\end{equation}
Plugging this into the final bound, we have:
\begin{equation}
    \frac{p}{T} \sum_{t=0}^{T-1} \mathbb{E}_{\mathcal{Z}_t}[\|\nabla f_\mu(\mathbf{x}_t)\|^2] + \frac{1-p}{T} \sum_{t=0}^{T-1} \mathbb{E}_{\mathcal{Z}_t}[\langle \nabla f_\mu(\mathbf{x}_t), \nabla \tilde{f}_\mu(\mathbf{x}_t)] \leq \sqrt{\frac{4 d^3 L_0(f)^5 R}{\epsilon^\prime T} \left(p + (1-p)\frac{L^2_0(\tilde{f})}{L^2_0(f)} + 2p(1-p)\frac{\epsilon^2}{\epsilon^{\prime 2}}\right)}
\end{equation}

This bound shows that optimizing the confusion term $\langle \nabla f_\mu(\mathbf{x}), \nabla \tilde{f}_\mu(\mathbf{x}) \rangle$ increases the bound and, consequently, increases the number of steps required by gradient estimation-based attacks.

\section{Theoretical Analysis of Gradient Confusion for Search-Based Attacks}
\label{app:search-based}
In this part, we provide the detailed proof of Theorem~\ref{thm:rs_switch_failure}, which theoretically shows that GC can be effective against search-based attacks by misleading the acceptance/rejection signal $h_t$.
As a result, the adversary requires more steps (and thus more queries) to craft an adversarial example.
We first present auxiliary lemmas used in the proof, and then provide the detailed proof of the theorem.

\subsection{Auxiliary Lemmas}

\begin{lemma}\label{lem:taylor}
Let $f:\mathbb{R}^d \to \mathbb{R}$ be $L$-smooth, i.e.,
\[
\|\nabla f(u) - \nabla f(v)\| \le L \|u - v\|
\quad \forall u,v \in \mathbb{R}^d.
\]
Then for any $x,\epsilon \in \mathbb{R}^d$,
\[
f(x+\epsilon)
= f(x) + \langle \nabla f(x), \epsilon \rangle + r_f,
\]
where the remainder term satisfies
\[
\big|r_f| \le \frac{L}{2}\|\epsilon\|^2.
\]
Equivalently,
\[
\big| f(x+\epsilon) - f(x) - \langle \nabla f(x), \epsilon \rangle \big|
\le \frac{L}{2}\|\epsilon\|^2.
\]
\end{lemma}

\begin{proof}
Define $\phi(t) = f(x + t\epsilon)$ for $t \in [0,1]$. By the fundamental theorem of calculus,
\[
f(x+\epsilon) - f(x)
= \int_0^1 \langle \nabla f(x+t\epsilon), \epsilon \rangle \, dt.
\]
Subtracting the first-order term yields
\[
f(x+\epsilon) - f(x)
= \langle \nabla f(x), \epsilon \rangle
+ \int_0^1 \langle \nabla f(x+t\epsilon) - \nabla f(x), \epsilon \rangle \, dt.
\]
Thus the remainder term is
\[
r_f
= \int_0^1 \langle \nabla f(x+t\epsilon) - \nabla f(x), \epsilon \rangle \, dt.
\]
Applying Cauchy--Schwarz and $L$-smoothness,
\[
|r_f|
\le \int_0^1 \|\nabla f(x+t\epsilon) - \nabla f(x)\|\,\|\epsilon\|\,dt
\le \int_0^1 Lt\|\epsilon\|^2\,dt
= \frac{L}{2}\|\epsilon\|^2.
\]
\end{proof}

\begin{lemma}\label{lem:gaussian-linear-quadratic}
Let $u \sim \mathcal{N}(0,I_d)$ in $\mathbb{R}^d$ with $d\ge 2$, and let $g\in\mathbb{R}^d$ be fixed with $g\neq 0$. Then for any $c\ge 0$,
\[
\mathbb{P}\!\left(\,|g^\top u| < c\,\|u\|^2\,\right)
\;\le\;
 \sqrt{\frac{2}{\pi}}\, \frac{c\, d}{\|g\|}.
\]
\end{lemma}

\begin{proof}
Write $u = r s$ where $r=\|u\|\ge 0$ and $s=u/\|u\|\in S^{d-1}$. By rotational invariance of the standard Gaussian, $r$ and $s$ are independent, $s$ is uniform on the unit sphere, and $r^2\sim\chi^2_d$. Let $a=\|g\|>0$. Then
\[
g^\top u = r\, g^\top s.
\]
By rotation symmetry, $g^\top s \stackrel{d}{=} aT$, where $T$ is the first coordinate of a uniform random point on $S^{d-1}$. Hence
\[
\mathbb{P}\!\left(|g^\top u| < c\|u\|^2 \,\middle|\, r\right)
= \mathbb{P}\!\left(|a r T| < c r^2 \,\middle|\, r\right)
= \mathbb{P}\!\left(|T| < \frac{c}{a}\,r\right).
\]
For $d\ge 2$, $T$ has density on $[-1,1]$
\[
f_T(t)=C_d(1-t^2)^{\frac{d-3}{2}},
\qquad
C_d=\frac{\Gamma(d/2)}{\sqrt{\pi}\,\Gamma((d-1)/2)}.
\]
Since $\sup_{t\in[-1,1]} f_T(t)=f_T(0)=C_d$, for any $t\ge 0$ we have
\[
\mathbb{P}(|T|<t)=\int_{-t}^{t} f_T(x)\,dx
\le 2t\,C_d.
\]
Therefore,
\[
\mathbb{P}\!\left(|T| < \frac{c}{a}\,r\right)
\le 2C_d\frac{c}{a}\,r.
\]
Taking expectation over $r$ gives
\[
\mathbb{P}\!\left(|g^\top u| < c\|u\|^2\right)
\le 2C_d\frac{c}{a}\,\mathbb{E}[r].
\]
By Cauchy--Schwarz, $\mathbb{E}[r]\le \sqrt{\mathbb{E}\|u\|^2}=\sqrt{d}$.
Finally, using $\Gamma(x+\tfrac12)/\Gamma(x)\le \sqrt{x}$ for all $x\ge \tfrac12$ with $x=\frac{d-1}{2}$,
\[
C_d=\frac{\Gamma(d/2)}{\sqrt{\pi}\,\Gamma((d-1)/2)}
=\frac{1}{\sqrt{\pi}}\frac{\Gamma(x+\tfrac12)}{\Gamma(x)}
\le \sqrt{\frac{d}{2\pi}}.
\]
Combining gives
\[
\mathbb{P}\!\left(|g^\top u| < c\|u\|^2\right)
\le 2\left(\sqrt{\frac{d}{2\pi}}\right)\frac{c}{\|g\|}\left(\sqrt{d}\right)
=\sqrt{\frac{2}{\pi}}\,\frac{c\,d}{\|g\|}.
\]
\end{proof}

\begin{lemma}\label{lem:sign_disagreement}
Let $u\sim \mathcal{N}(0,I_d)$ and let $g,\tilde g\in\mathbb{R}^d$ be nonzero vectors. Define $\theta\in[0,\pi]$ to be the angle between $g$ and $\tilde g$, i.e.
\[
\cos\theta \;=\; \frac{\langle g,\tilde g\rangle}{\|g\|\,\|\tilde g\|}.
\]
Then
\[
\mathbb{P}\Big(\sign(\langle g,u\rangle)\neq \sign(\langle \tilde g,u\rangle)\Big)
\;=\;
\frac{\theta}{\pi}
=
\frac{1}{\pi}\arccos\!\left(\frac{\langle g,\tilde g\rangle}{\|g\|\,\|\tilde g\|}\right).
\]
\end{lemma}

\begin{proof}
Set $X=\langle g,u\rangle$ and $Y=\langle \tilde g,u\rangle$. Then $(X,Y)$ is a centered bivariate Gaussian with
\[
\operatorname{Var}(X)=\|g\|^2,\quad \operatorname{Var}(Y)=\|\tilde g\|^2,\quad
\operatorname{Cov}(X,Y)=\langle g,\tilde g\rangle,
\]
so the correlation is $\rho=\langle g,\tilde g\rangle/(\|g\|\|\tilde g\|)=\cos\theta$.
A standard orthant probability identity gives
\[
\mathbb{P}(X>0,Y>0)=\frac14+\frac{1}{2\pi}\arcsin(\rho).
\]
By symmetry,
\[
\mathbb{P}(XY<0)=1-2\mathbb{P}(X>0,Y>0)
=\frac12-\frac{1}{\pi}\arcsin(\rho)
=\frac{1}{\pi}\arccos(\rho)
=\frac{\theta}{\pi}.
\]
\end{proof}

\subsection{Proof of Theorem~\ref{thm:rs_switch_failure}}
With the above lemmas in place, we now state the theorem and present a detailed proof.

\begin{theorem}[Theorem~\ref{thm:rs_switch_failure} (restated)]
Under Assumption~\ref{assum:smooth}, for $d\ge 2$, $\mathbf{u}\sim\mathcal N(0,I_d)$, fixed $\mathbf{x}\in\mathbb R^d$, $\mu>0$,
and with per-query switching defense $f_{\mathrm{def}}$, the signals
\[
h := f(\mathbf{x}+\mu\mathbf{u})-f(\mathbf{x}),
\qquad
h_{\mathrm{def}} := f_{\mathrm{def}}(\mathbf{x}+\mu\mathbf{u})-f_{\mathrm{def}}(\mathbf{x})
\]
satisfy the bound in~\eqref{eq:rs_switch_failure_bound}.
\end{theorem}

\begin{proof}
Let
\[
\Delta_f := f(x+\mu u)-f(x),
\qquad
\Delta_{\tilde f} := \tilde f(x+\mu u)-\tilde f(x).
\]
By definition, $h=\Delta_f$. Let $\lambda_1,\lambda_2\stackrel{\text{i.i.d.}}{\sim}\mathrm{Bernoulli}(p)$ be independent of $u$, and write
\[
f_{\mathrm{def}}(x+\mu u)=\lambda_1 f(x+\mu u)+(1-\lambda_1)\tilde f(x+\mu u),\quad
f_{\mathrm{def}}(x)=\lambda_2 f(x)+(1-\lambda_2)\tilde f(x).
\]
Define $E:=\{\sign(h)\neq \sign(h_{\mathrm{def}})\}$. Conditioning on $(\lambda_1,\lambda_2)$ and keeping only the case $(0,0)$ gives
\begin{align*}
\Pp(E)
&\ge \Pp\big(E\cap\{\lambda_1=0,\lambda_2=0\}\big) \\
&= \Pp(\lambda_1=0,\lambda_2=0)\;
\Pp\!\left(\sign(\Delta_f)\neq \sign(\Delta_{\tilde f})\right) \\
&=(1-p)^2\;
\Pp\!\left(\sign(\Delta_f)\neq \sign(\Delta_{\tilde f})\right).
\end{align*}
It remains to lower bound $\Pp(\sign(\Delta_f)\neq \sign(\Delta_{\tilde f}))$.

Define the linear parts
\[
A := \langle \nabla f(x),u\rangle,
\qquad
\tilde A := \langle \nabla \tilde f(x),u\rangle.
\]
By Lemma~\ref{lem:taylor} applied to $f$ with $L=L_1(f)$ and $\epsilon=\mu u$,
\[
\Delta_f = \mu A + r_f,
\qquad
|r_f|\le \frac{L_1(f)}{2}\mu^2\|u\|^2.
\]
Similarly, applying Lemma~\ref{lem:taylor} to $\tilde f$ with $L=L_1(\tilde f)$,
\[
\Delta_{\tilde f} = \mu \tilde A + r_{\tilde f},
\qquad
|r_{\tilde f}|\le \frac{L_1(\tilde f)}{2}\mu^2\|u\|^2.
\]

\paragraph{Step 1: sign-stability events.}
If $|\delta|<|a|$, then $\sign(a+\delta)=\sign(a)$. Define
\[
G_f := \left\{\frac{L_1(f)}{2}\mu\|u\|^2<|A|\right\},
\qquad
G_{\tilde f} := \left\{\frac{L_1(\tilde f)}{2}\mu\|u\|^2<|\tilde A|\right\}.
\]
On $G_f$, we have $|r_f|<\mu|A|$, hence $\sign(\Delta_f)=\sign(A)$; similarly on $G_{\tilde f}$,
$\sign(\Delta_{\tilde f})=\sign(\tilde A)$. Therefore,
\[
\Pp\!\left(\sign(\Delta_f)\neq \sign(\Delta_{\tilde f})\right)
\ge
\Pp\!\left(\sign(A)\neq \sign(\tilde A)\right)
-\Pp(G_f^c)-\Pp(G_{\tilde f}^c).
\]

\paragraph{Step 2: linear sign disagreement.}
By Lemma~\ref{lem:sign_disagreement} with $g=\nabla f(x)$ and $\tilde g=\nabla \tilde f(x)$,
\[
\Pp\!\left(\sign(A)\neq \sign(\tilde A)\right)
=
\frac{1}{\pi}\arccos\!\left(
\frac{\langle \nabla f(x),\nabla \tilde f(x)\rangle}{\|\nabla f(x)\|\,\|\nabla \tilde f(x)\|}
\right).
\]

\paragraph{Step 3: bound Taylor-failure probabilities.}
Using Lemma~\ref{lem:gaussian-linear-quadratic} with $g=\nabla f(x)$ and $c=\frac{L_1(f)}{2}\mu$,
\[
\Pp(G_f^c)
=
\Pp\!\left(|A|<\tfrac{L_1(f)}{2}\mu\|u\|^2\right)
\le
\sqrt{\frac{2}{\pi}}\frac{(\tfrac{L_1(f)}{2}\mu)d}{\|\nabla f(x)\|}
=
\frac{\mu d}{\sqrt{2\pi}}\frac{L_1(f)}{\|\nabla f(x)\|}.
\]
Similarly,
\[
\Pp(G_{\tilde f}^c)
\le
\frac{\mu d}{\sqrt{2\pi}}\frac{L_1(\tilde f)}{\|\nabla \tilde f(x)\|}.
\]

\paragraph{Step 4: combine and multiply by $(1-p)^2$.}
Combining Steps 1--3 yields
\[
\Pp\!\left(\sign(\Delta_f)\neq \sign(\Delta_{\tilde f})\right)
\ge
\frac{1}{\pi}\arccos\!\left(
\frac{\langle \nabla f(x),\nabla \tilde f(x)\rangle}{\|\nabla f(x)\|\,\|\nabla \tilde f(x)\|}
\right)
-
\frac{\mu d}{\sqrt{2\pi}}
\left(
\frac{L_1(f)}{\|\nabla f(x)\|}
+
\frac{L_1(\tilde f)}{\|\nabla \tilde f(x)\|}
\right).
\]
Multiplying by $(1-p)^2$ gives~\eqref{eq:rs_switch_failure_bound}.
\end{proof}

\subsection{Why Gaussian Directions?}
\label{app:gaussian}

We model the search direction as Gaussian because it is rotationally invariant and makes the analysis depend only on geometry. This is also a good approximation for many isotropic samplers used in practice: in high dimensions, one-dimensional projections (i.e., inner products with a random direction) of a broad class of high-dimensional random vectors are close to Gaussian under mild conditions, so Gaussian modeling captures the sign behavior of these projections well~\citep{meckes2012approximation}. In the proof, the key quantity is the probability that the two linearized one-step signals induced by $\nabla f(\mathbf{x})$ and $\nabla\tilde f(\mathbf{x})$ disagree in sign, and rotational invariance implies this probability depends only on their relative orientation. This is exactly where the arccosine term in \eqref{eq:rs_switch_failure_bound} comes from.



\end{document}